\documentclass[final,12pt]{colt2022} 
\pdfoutput=1

\title[Chasing Convex Bodies and Functions with Black-Box Advice]{Chasing Convex Bodies and Functions with Black-Box Advice}
\usepackage{times}
\usepackage{amsmath, amssymb}
\usepackage{amsfonts}
\usepackage{bm}
\usepackage{enumerate}
\usepackage[shortlabels]{enumitem}
\usepackage{mathtools}
\usepackage{graphicx}
\usepackage{titlesec}
\usepackage{braket}
\usepackage{cleveref}
\usepackage{tikz-cd}
\tikzset{%
  symbol/.style={
    draw=none,
    every to/.append style={
      edge node={node [sloped, allow upside down, auto=false]{$#1$}}
    },
  },
}
\usepackage{float}
\usepackage{bbm}
\usepackage{mathrsfs}
\usepackage[ruled, lined, linesnumbered]{algorithm2e}
\usepackage{hyperref}
\usepackage{xcolor}
\hypersetup{
    colorlinks=true,
    linkcolor=blue,
    filecolor=magenta,
    urlcolor=cyan,
}

\newcommand{\R}{{\mathbb R}}
\newcommand{\N}{{\mathbb N}}

\DeclareMathOperator{\aff}{aff}
\DeclareMathOperator{\sgn}{sgn}

\DeclareMathOperator*{\argmax}{arg\,max}

\newcommand{\bv}[1]{\mathbf{#1}}

\newcommand{\calA}{{\mathcal{A}}}

\newcommand{\calX}{{\mathcal{X}}}

\newcommand{\calO}{{\mathcal{O}}}

\newcommand{\cbc}{{\mathsf{CBC}}}
\newcommand{\ncbc}{{\mathsf{NCBC}}}
\newcommand{\kcbc}{{k\mathsf{CBC}}}
\newcommand{\cfc}{{\mathsf{CFC}}}
\newcommand{\soco}{{\mathsf{SOCO}}}
\newcommand{\acfc}{{\mathsf{\alpha CFC}}}

\newcommand{\kgcfc}{{\mathsf{(\kappa,\gamma) CFC}}}
\newcommand{\rob}{{\textsc{Rob}}}
\newcommand{\opt}{{\textsc{Opt}}}
\newcommand{\adv}{{\textsc{Adv}}}
\newcommand{\alg}{{\textsc{Alg}}}
\newcommand{\meta}{{\textsc{Meta}}}
\newcommand{\algone}{\textsc{Alg}^{(1)}}
\newcommand{\algtwo}{\textsc{Alg}^{(2)}}

\newcommand{\switch}{{\textsc{Switch}}}
\newcommand{\nestedswitch}{{\textsc{NestedSwitch}}}

\newcommand{\interp}{{\textsc{Interp}}}
\newcommand{\binterp}{{\textsc{BdInterp}}}
\newcommand{\ns}{{\textsc{NS}}}
\newcommand{\avec}{{\vec{a}}}
\newcommand{\bvec}{{\vec{b}}}
\newcommand{\cvec}{{\vec{c}}}
\newcommand{\hvec}{{\vec{h}}}
\newcommand{\svec}{{\vec{s}}}
\newcommand{\rvec}{{\vec{r}}}
\newcommand{\vvec}{{\vec{v}}}
\newcommand{\wvec}{{\vec{w}}}
\newcommand{\xvec}{{\vec{x}}}
\newcommand{\yvec}{{\vec{y}}}
\newcommand{\zvec}{{\vec{z}}}
\newcommand{\evec}{{\vec{e}}}
\newcommand{\cost}{{\mathrm{C}}}
\newcommand*\tageq{\refstepcounter{equation}\tag{\theequation}}

\coltauthor{%
 \Name{Nicolas Christianson} \Email{nchristianson@caltech.edu}\\
 \Name{Tinashe Handina} \Email{thandina@caltech.edu}\\
 \Name{Adam Wierman} \Email{adamw@caltech.edu} \\
 \addr California Institute of Technology}

\begin{document}

\maketitle

\begin{abstract}%
    We consider the problem of convex function chasing with black-box advice, where an online decision-maker aims to minimize the total cost of making and switching between decisions in a normed vector space, aided by black-box advice such as the decisions of a machine-learned algorithm. The decision-maker seeks cost comparable to the advice when it performs well, known as \emph{consistency}, while also ensuring worst-case \emph{robustness} even when the advice is adversarial. We first consider the common paradigm of algorithms that switch between the decisions of the advice and a competitive algorithm, showing that no algorithm in this class can improve upon 3-consistency while staying robust. We then propose two novel algorithms that bypass this limitation by exploiting the problem's convexity. The first, $\interp$, achieves $(\sqrt{2}+\epsilon)$-consistency and $\calO(\frac{C}{\epsilon^2})$-robustness for any $\epsilon > 0$, where $C$ is the competitive ratio of an algorithm for convex function chasing or a subclass thereof. The second, $\binterp$, achieves $(1+\epsilon)$-consistency and $\calO(\frac{CD}{\epsilon})$-robustness when the problem has bounded diameter $D$. Further, we show that $\binterp$ achieves near-optimal consistency-robustness trade-off for the special case where cost functions are $\alpha$-polyhedral.
\end{abstract}


\section{Introduction}

We study the problem of convex function chasing ($\cfc$), in which a player chooses decisions $\xvec_t$ online from a normed vector space $\calX = (X, \|\cdot\|)$ in order to minimize the total cost $\sum_{t=1}^T f_t(\xvec_t) + \|\xvec_t - \xvec_{t-1}\|$, where each $f_t$ is a convex ``hitting'' cost function that is revealed prior to the player's selection of $\xvec_t$, and the term $\|\xvec_t - \xvec_{t-1}\|$ penalizes changing decisions between rounds. A number of subclasses of $\cfc$ have been discussed in the literature, characterized by various restrictions on the class of cost functions $f_t$. Of particular note is the special case of convex body chasing ($\cbc$), in which each cost function $f_t$ is the $\{0, \infty\}$ indicator of a convex set $K_t$, so that each decision $\xvec_t$ must reside strictly within $K_t$. Algorithms for $\cfc$ and its special cases are judged on the basis of their competitive ratio, i.e., the worst-case ratio in cost between the algorithm and the hindsight optimal sequence of decisions (Definition \ref{defn:competitive_ratio}).

Convex body chasing and function chasing were introduced by \cite{friedman_convex_1993} as continuous versions of several fundamental problems in online algorithms, including Metrical Task Systems (\cite{borodin_optimal_1992}) and the $k$-server problem (\cite{koutsoupias_k-server_1995}). $\cfc$ has also been studied recently as the problem of ``smoothed online convex optimization'' ($\soco$), introduced by \cite{lin_online_2012}. The basic premise of $\cfc$/$\soco$, of choosing decisions online to optimize per-round costs with minimal movement between decisions, has seen wide application in a number of domains, including data center load-balancing (\cite{lin_online_2012}) and right-sizing (\cite{lin_dynamic_2013, albers_optimal_2018}), electric vehicle charging (\cite{kim_real-time_2014}), and control (\cite{goel_online_2019, li_online_2021}). 

In high-dimensional settings, the performance of algorithms for $\cbc$ and $\cfc$ can be arbitrarily poor: \cite{friedman_convex_1993} showed a $\sqrt{d}$ lower bound on the competitive ratio of any algorithm for $\cbc$ (and thus $\cfc$) in $d$-dimensional Euclidean space, which \cite{bubeck_chasing_2019} extended to an $\Omega(\max\{\sqrt{d}, d^{1-\frac{1}{p}}\})$ lower bound in $\R^d$ with the $\ell^p$ norm. Prospects are poor even for subclasses of $\cfc$ with additional restrictions on the functions $f_t$. For instance, $\cfc$ with \emph{$\alpha$-polyhedral} cost functions, i.e. where each $f_t$ has a unique minimizer away from which it grows with slope at least $\alpha > 0$, has been studied widely in the $\soco$ literature. State-of-the-art algorithms in this setting achieve competitive ratio $\calO(\alpha^{-1})$, which grows arbitrarily large in the $\alpha \to 0$ limit (\cite{chen_smoothed_2018,zhang_revisiting_2021}).

The modern tools of machine learning wield great promise for improving upon these pessimistic performance guarantees. That is, for practical applications, there is often large amounts of data recorded from past problem instances, enabling the training of machine learning models that can outperform traditional, conservative online algorithms. However, these machine-learned algorithms are ``black boxes,'' in the sense that they lack rigorous, worst-case performance guarantees. Such black-box algorithms might \emph{typically} outperform robust online algorithms, but their lack of uncertainty quantification can lead to arbitrarily poor performance in the worse case, if they are deployed on held-out problem instances or under distribution shift.

Thus, a natural question arises: \emph{is it possible to develop algorithms that achieve both the worst-case guarantees of traditional online algorithms for $\cfc$ and the average-case performance of machine-learned algorithms or other sources of black-box ``advice''?} 

These desiderata are naturally encoded in the notions of \emph{robustness} and \emph{consistency} introduced by \cite{lykouris_competitive_2018} in the context of competitive caching. In this framework, a \emph{consistent} algorithm is one with a competitive ratio with respect to the black-box advice, implying that when the advice is accurate, the algorithm will perform well; on the other hand, a \emph{robust} algorithm is one that has a finite competitive ratio, regardless of advice performance. Our goal is to develop algorithms with tunable robustness and consistency guarantees, so that a decision-maker can decide in advance the trade-off they wish to make between exploiting good advice performance and ensuring worst-case robustness in the case that advice performs poorly.

\subsection{Contributions} We answer the question above by proposing novel algorithms with tunable robustness and consistency bounds for $\cfc$ and any subclass thereof. In particular, we reduce the problem of designing robust and consistent algorithms for $\cfc$ to the design of \emph{bicompetitive meta-algorithms} (Definitions \ref{def:bicompetitive}, \ref{def:metaalgorithm}), which are unified ``recipes'' for combining black-box advice with a robust algorithm in a manner that guarantees a competitive ratio with respect to both ingredients. These ``recipes'' are very general -- they can be used to combine advice with \emph{any} algorithm for any subclass of $\cfc$ to obtain a customized robustness and consistency guarantee for that subclass without explicit knowledge of the algorithm or advice design. 

More specifically, our contributions are twofold. We first consider the class of ``switching'' algorithms, which switch between the decisions of the advice and a robust algorithm. This class of algorithms has received considerable attention in the literature on robustness and consistency, and in particular, all prior algorithms for $\cfc$ with black-box advice in dimension greater than one have been switching algorithms. We prove a fundamental limit on the robustness and consistency of any switching algorithm for $\cfc$, showing that \emph{no} switching algorithm for $\cfc$ can improve on $3$-consistency while obtaining finite robustness (Theorem \ref{theorem:switching_lowerbound}). We give a switching meta-algorithm $\switch$ (Appendix \ref{appendix:switch_algorithm}, Algorithm \ref{alg:switch}) achieving this fundamental limit, obtaining $(3+\calO(\epsilon))$-consistency and $\calO(\frac{C}{\epsilon^2})$-robustness for any $\epsilon > 0$, where $C$ is the competitive ratio of any algorithm for $\cfc$ or a subclass thereof. We further show that the fundamental limit on switching algorithms can be broken in the special case of nested $\cbc$, in which successive bodies are nested. In this setting, we provide an algorithm $\nestedswitch$ (Algorithm \ref{alg:nested_metaalg}) achieving $(1+\epsilon)$-consistency along with $\calO(\frac{d}{\epsilon})$-robustness for nested $\cbc$ in $d$ dimensions (Proposition \ref{proposition:switching_steiner}).

\emph{Second}, galvanized by the limitations of switching algorithms, we develop algorithms exploiting the convexity of the $\cfc$ problem to obtain improved robustness and consistency bounds. We propose a meta-algorithm $\interp$ (Algorithm \ref{alg:interp}) that, given a $C$-competitive algorithm for a subclass of $\cfc$, achieves $(\mu(\calX) + \epsilon)$-consistency and $\calO(\frac{C}{\epsilon^2})$-robustness for any desired $\epsilon > 0$, where $\mu(\calX)$ is a geometric constant depending on the structure of the normed space $\calX$ that is $\sqrt{2}$ in any Hilbert space and is strictly less than $3$ in any $\ell^p$ space, $p \in (1, \infty)$ (Theorem \ref{theorem:consistent_robust_general}). Moreover, under the additional assumption that the advice and the $C$-competitive algorithm are never farther apart than some distance $D$, we give a meta-algorithm $\binterp$ (Algorithm \ref{alg:binterp}) that achieves $(1+\epsilon)$-consistency and $\calO(\frac{CD}{\epsilon})$-robustness (Theorem \ref{theorem:bounded_bicompetitive}). In particular, $\binterp$ gives nearly-optimal consistency and robustness for the problem of $\cfc$ with $\alpha$-polyhedral cost functions when $D = O(1)$ in $\alpha$.

A key feature of our results is their generality: our main results on bicompetitive meta-algorithms (Proposition \ref{prop:3_switching_alg}, Theorems \ref{theorem:consistent_robust_general}, \ref{theorem:bounded_bicompetitive}) hold in vector spaces with any norm and arbitrary, even infinite, dimension. This enables application to problems where the decisions $\xvec_t$ are infinite-dimensional objects such as probability measures, which could arise in settings such as iterated games. Moreover, these meta-algorithms enable the design of customized robust and consistent algorithms for \emph{any} subclass of $\cfc$, since they are agnostic to the specific algorithms used. We illustrate this by giving specific robustness and consistency results for the cases of $\cfc$, $\cbc$, and $\cfc$ restricted to $\alpha$-polyhedral hitting cost functions; we give further examples in Appendix \ref{appendix:laundry_list}.

\subsection{Related work}

Our work contributes to the literatures on $\cbc$, $\cfc$, and $\soco$ as well as the emerging literature on online algorithms with black-box advice.  We discuss each in turn below.

\subsubsection{$\cbc$, $\cfc$, and $\soco$}
The problems of convex body chasing and function chasing were introduced by \cite{friedman_convex_1993}, who gave a competitive algorithm for $\cbc$ in 2-dimensional Euclidean space. The problem in general dimension $d$ has been largely settled in the last few years.  In the setting where subsequent bodies are nested, \cite{argue_nearly-linear_2019} gave an $\calO(d\log d)$-competitive algorithm in any norm, and \cite{bubeck_chasing_2019} later gave an $\calO(\min\{d, \sqrt{d\log T}\})$-competitive algorithm in the Euclidean setting that uses the geometric Steiner point of the convex bodies. Later, \cite{argue_chasing_2021} and \cite{sellke_chasing_2020} concurrently obtained $\calO(d)$-competitive algorithms for general $\cfc$. The latter work builds upon the methods of \cite{bubeck_chasing_2019}, developing a ``functional'' Steiner point algorithm that is $d$-competitive for $\cbc$ and $(d+1)$-competitive for $\cfc$ in any normed space, matching the lower bound of $d$ in the $\ell^\infty$ norm setting.

Several special cases of $\cfc$/$\soco$ with restrictions on hitting cost structure have been studied in the literature to the end of obtaining ``dimension-free'' competitive ratios for these subclasses. \cite{chen_smoothed_2018} obtained the first such bound for the subclass of $\cfc$ where hitting cost functions $f_t$ are $\alpha$-polyhedral, which we call $\acfc$. The authors propose an algorithm, ``Online Balanced Descent'' (OBD), which achieves a competitive ratio $\calO(\frac{1}{\alpha})$. This upper bound has been successively refined, with the most recent entry a simple greedy algorithm from \cite{zhang_revisiting_2021} that achieves competitive ratio $\max\{1, \frac{2}{\alpha}\}$ in any normed vector space of arbitrary (even infinite) dimension by moving to the minimizer of each hitting cost function. The $\calO(\frac{1}{\alpha})$ upper bound has been broken by \cite{lin_personal_2022} in the finite-dimensional Euclidean setting with an $\calO(\frac{1}{\alpha^{1/2}})$-competitive algorithm, Greedy OBD, that is optimal within the class of memoryless, rotation- and scale-invariant algorithms. Another subclass of $\cfc$ that has received attention is that with $(\kappa, \gamma)$-well-centered hitting cost functions, which generalize well-conditioned functions. \cite{argue_dimension-free_2020} propose an algorithm achieving an $\calO(2^{\gamma/2}\kappa)$ competitive ratio for this subclass, and an improved algorithm achieving competitive ratio $\calO(\sqrt{\kappa})$ for the particular class of $\kappa$-well-conditioned functions along with a nearly matching $\Omega(\kappa^{1/3})$ lower bound. We summarize the state-of-the-art algorithms and competitive ratios that we refer to in our later results in Table \ref{table:sota}, giving an extended version of the table in Appendix \ref{appendix:prelim}, Table \ref{table:sota_extended}.

\begin{table}
\begin{center}
\renewcommand*{\arraystretch}{1.2}
\begin{tabular}{|c||c|c|c|} 
  \hline
  Problem & Algorithm Name & Competitive Ratio & Setting \\
  \hline\hline
  $\cfc$ & Functional Steiner Point & $d+1$ & $\R^d$ with any norm \\
  \hline 
  $\cbc$ & Functional Steiner Point & $d$ & $\R^d$ with any norm \\
  \hline 
  $\acfc$ & Greedy & $\max\left\{1, \frac{2}{\alpha}\right\}$ & Any normed vector space \\
  \hline
  $\acfc$ & Greedy OBD & $\calO\left(\frac{1}{\alpha^{1/2}}\right)$ & $\R^d$ with $\ell^2$ norm \\
  \hline
\end{tabular}
\caption{Competitive ratios for state-of-the-art algorithms for $\cfc, \cbc$, and $\acfc$.}
\label{table:sota}
\end{center}
\end{table}

\subsubsection{Online Algorithms with Black-Box Advice}
The idea of using machine-learned, black-box advice to improve online algorithms was first proposed by \cite{mahdian_online_2012} to design algorithms for online ad allocation, load balancing, and facility location. Formal notions of \emph{robustness} and \emph{consistency} were later coined by \cite{lykouris_competitive_2018} in the context of designing learning-augmented algorithms for caching. The last few years have seen a surge in the application of the robustness and consistency paradigm in designing online algorithms augmented with black-box advice for a multitude of problems, for example ski rental and non-clairvoyant scheduling (\cite{purohit_improving_2018, NEURIPS2020_5bd844f1}), energy generation scheduling (\cite{lee_online_2021}), bidding and bin-packing (\cite{angelopoulos_online_2020}), and Q-learning (\cite{golowich_can_2021}).

Closest to our work are the recent papers of \cite{antoniadis_online_2020} and \cite{rutten_online_2022}. The former considers the problem of designing algorithms for \emph{metrical task systems} (MTS) with black-box advice. MTS can be thought of as (non-convex) function chasing on general metric spaces, and hence their results also give robustness/consistency guarantees for $\cfc$. They apply two classical results on combining $k$-server algorithms (\cite{fiat_competitive_1994, baeza-yates_searching_1993}) and on combining MTS algorithms via $k$-experts algorithms (\cite{blum_-line_1997, freund_decision-theoretic_1997}) to devise, in our parlance, bicompetitive meta-algorithms for MTS. In particular, their first algorithm switches between the advice and a $C$-competitive algorithm for MTS, and achieves $9$-consistency and $9C$-robustness. They also propose a randomized switching algorithm that, under the assumption that the metric space has bounded diameter $D$, obtains cost bounded in expectation by $\min\{(1+\epsilon)\cost_{\adv} + \calO(\frac{D}{\epsilon}), (1+\epsilon)C \cdot \cost_\opt + \calO(\frac{D}{\epsilon})\}$, where $\cost_\adv$ is the cost of the advice and $\cost_\opt$ is the optimal cost. However, the large $\calO(\frac{D}{\epsilon})$ additive factors in their result preclude $(1+\epsilon)$-consistency, since when $\cost_\adv = \calO(1)$, the consistency bound will be $1 + \epsilon + \Omega(\frac{D}{\epsilon})$. This is to be expected, since as we show in Section \ref{section:switching}, no deterministic switching algorithm can improve on 3-consistency while having finite robustness. Moreover, their results due not allow tuning robustness and consistency, i.e., neither algorithm allows trading-off robustness in order to obtain consistency arbitrarily close to 1.

On the other hand, \cite{rutten_online_2022} considers the problem of $\cfc$ with $\alpha$-polyhedral hitting costs ($\acfc$), but with the convexity assumption dropped from the hitting costs $f_t$. They obtain a $(1+\epsilon)$-consistent, $2^{\tilde{\calO}(\frac{1}{\alpha\epsilon})}$-robust algorithm in this setting, together with a lower bound showing that this exponential trade-off between robustness and consistency is necessary due to their non-convex setting. Their algorithm is a switching algorithm and crucially depends on the $\alpha$-polyhedral structure of the hitting cost functions, and hence cannot be extended to general $\cfc$. The authors also propose an algorithm for $\cfc$ in the 1-dimensional case (where $X = \R$) that achieves $(1+\epsilon)$-consistency and $\calO(\frac{1}{\epsilon^2})$ robustness, and they prove a lower bound of $(1+\epsilon)$-consistency and $\calO(\frac{1}{\epsilon})$ robustness on \emph{any} algorithm for $\cfc$ with black-box advice. They leave open the broader problem of developing robust and consistent algorithms for $\cfc$ and its many subclasses in the higher-dimensional setting.

\subsection{Notation} Throughout this paper $X$ refers to a real vector space of arbitrary dimension. When a norm $\|\cdot\|$ is distinguished, $B(\xvec, r)$ is the closed $\|\cdot\|$-ball of radius $r \geq 0$ centered at $\xvec$, and $\Pi_K\xvec$ is a metric projection of the point $\xvec \in X$ onto a closed convex set $K$. For $\xvec, \yvec \in X$, we define $[\xvec, \yvec] \coloneqq \{\vec{z} \in X : \vec{z} = \lambda\xvec + (1-\lambda)\yvec, \lambda \in [0, 1]\}$ as the convex span of $\xvec$ and $\yvec$. The non-negative reals are denoted by $\R_+$, and for $T \in \N$, we write $[T] \coloneqq \{1, \ldots, T\}$. Asymptotic notation involving the variable $\epsilon > 0$ reflects the asymptotic regime $\epsilon \to 0$.

\section{Preliminaries}



We consider the general problem of \emph{convex function chasing} ($\cfc$) on a real normed vector space $\calX = (X, \|\cdot\|)$. In particular, we make no assumption on either the dimension of $\calX$ or on the choice of norm $\|\cdot\|$. In $\cfc$, a decision-maker begins at some initial point $\xvec_0 \in X$, and at each time $t \in \N$ is handed a convex function $f_t : X \to \R_+$ and must choose some $\xvec_t \in X$, paying both the \emph{hitting cost} $f_t(\xvec_t)$, as well as the \emph{movement} or \emph{switching cost} $\|\xvec_t - \xvec_{t-1}\|$ induced by the norm. Crucially, $\xvec_t$ is chosen prior to the revelation of any future cost functions $f_k$, $k > t$, i.e., decisions are made \emph{online}. The game ends at some time $T \in \N$, which is unknown to the decision-maker in advance. We refer to a tuple $(\xvec_0, f_1, \ldots, f_T)$ as an \emph{instance} of the $\cfc$ problem. The total cost incurred by the decision-maker on a problem instance is $\sum_{t=1}^T f_t(\xvec_t) + \|\xvec_t - \xvec_{t-1}\|$.

Informally, an \emph{online algorithm} for $\cfc$ is an algorithm that, on a given instance of $\cfc$, produces decisions online. We denote by $\alg_t$ the $t$\textsuperscript{th} decision made by an online algorithm $\alg$; by convention, $\alg_0 \coloneqq \xvec_0$, the starting point of the instance. Then the cost $\cost_\alg$ incurred by $\alg$ on an instance is
$$\cost_\alg = \sum_{t=1}^T f_t(\alg_t) + \|\alg_t - \alg_{t-1}\|.$$
We also introduce the partial cost notation $\cost_\alg(t, t') = \sum_{i=t}^{t'} f_i(\alg_i) + \|\alg_i - \alg_{i-1}\|$, defined for $1 \leq t \leq t' \leq T$. We refer to the set of all online algorithms for $\cfc$ as $\calA_\cfc$.


We typically compare online algorithms for $\cfc$ against $\opt$, the offline optimal algorithm that chooses the hindsight optimal sequence of decisions for any problem instance. Its cost is the optimal value of the following convex program:
$$\cost_\opt = \cost_\opt(\xvec_0, f_1, \ldots, f_T) \coloneqq \min_{\xvec_1, \ldots, \xvec_T \in X} \sum_{t=1}^T f_t(\xvec_t) +  \|\xvec_t - \xvec_{t-1}\|$$
and its decisions are determined by the optimal solution. To evaluate the performance of an online algorithm for $\cfc$, we consider the \emph{competitive ratio}, which measures the worst case ratio in costs between an algorithm and $\opt$. In the following, we define both the conventional competitive ratio as well as a generalization that allows for comparing against arbitrary benchmark algorithms.

\begin{definition} \label{defn:competitive_ratio}
    Let $\algone$ be an online algorithm for $\cfc$, and let $\algtwo$ be another (not necessarily online) algorithm for $\cfc$.\footnote{Like $\opt$, the decision of $\algtwo$ at some time $t$ is allowed to depend on problem instance data revealed after time $t$.} $\algone$ is defined to be \textbf{$C$-competitive with respect to} $\algtwo$ if, regardless of problem instance,
    $\cost_\algone \leq C \cdot \cost_\algtwo.$
    In particular, if $\algtwo = \opt$, we simply say that $\algone$ is \textbf{$C$-competitive}, or has \textbf{competitive ratio $C$}.
\end{definition}

\subsection{Subclasses of Convex Function Chasing \label{section:cfc_subclasses}}
$\cfc$ is a broad set of problems and many subclasses have received attention in the literature. We consider several subclasses of the general $\cfc$ problem in this work, distinguished by different assumptions on the hitting cost functions. In this section, we briefly define the subclasses of $\cfc$ which we refer to in our later results in the main text. We give more detailed definitions of these and several other subclasses of $\cfc$ in Appendix \ref{appendix:cfc_subclasses}.

\subsubsection{Convex Body Chasing}
In the problem of convex body chasing ($\cbc$), the decision-maker must choose each decision $\xvec_t$ from a convex body $K_t \subseteq X$ that is revealed online. This can be seen as a special case of $\cfc$ where $f_t$ is $0$ on $K_t$ and $\infty$ elsewhere; see Appendix \ref{appendix:cbc_cfc_equiv} for more details on this equivalence. A notable special case of $\cbc$ is the problem of \emph{nested} convex body chasing ($\ncbc$), in which subsequent bodies are nested, i.e., $K_t \supseteq K_{t+1}$ for each $t$. We define $\calA_\cbc$ as the set of all online algorithms for $\cbc$ that are \emph{feasible}, i.e., that produce decisions within the convex body $K_t$ at each time. We define $\calA_\ncbc$ similarly as the set of \emph{feasible} online algorithms for $\ncbc$.

\subsubsection{$\alpha$-Polyhedral Convex Function Chasing}
Several subclasses of $\cfc$ have been studied in the literature with hitting cost functions $f_t$ restricted so as to enable dimension-free competitive ratios. One of the most well-studied such subclasses is the problem of $\alpha$-polyhedral convex function chasing ($\acfc$), e.g. \cite{chen_smoothed_2018, zhang_revisiting_2021}, in which each hitting cost function $f_t$ is restricted to be globally $\alpha$-polyhedral, meaning intuitively that it has a unique minimizer, away from which it grows with slope at least $\alpha > 0$.

\begin{definition}
    Let $(X, \|\cdot\|)$ be a normed vector space, and let $\alpha > 0$. A function $f : X \to \R_+$ is \textbf{globally $\alpha$-polyhedral} if it has unique minimizer $\xvec^* \in X$, and in addition,
    $$f(\xvec) \geq f(\xvec^*) +  \alpha\|\xvec - \xvec^*\| \qquad\text{for all $\xvec \in X$}.$$
\end{definition}

\subsection{Using Black-Box Advice: Robustness, Consistency, and Bicompetitive Analysis \label{section:robust_consistent_intro}}

In this work, we seek algorithms for $\cfc$ and its subclasses that can exploit the good performance of a black-box advice algorithm, such as a reinforcement learning model, while maintaining rigorous worst-case performance guarantees. More specifically, we strive for algorithms that can obtain cost not much worse than optimal when the black-box advice is perfect, yet which have uniformly bounded competitive ratio when the advice is arbitrarily bad or even adversarial. This dual objective is naturally formulated in terms of \emph{robustness} and \emph{consistency}, which were introduced by \cite{lykouris_competitive_2018} and are defined as follows.

\begin{definition}
    Let $\alg$ be an online algorithm for $\cfc$, and let $\adv$ be a black-box advice algorithm. $\alg$ is said to be \textbf{$c$-consistent} if it is $c$-competitive with respect to $\adv$. On the other hand, $\alg$ is defined to be \textbf{$r$-robust} if it is $r$-competitive, independent of the performance of $\adv$.
\end{definition}

Our precise goal is to design algorithms achieving $(1+\epsilon)$-consistency and $R(\epsilon)$-robustness for $\cfc$ and its subclasses, where $\epsilon > 0$ is a hyperparameter chosen by the decision-maker that encodes confidence in the advice. The dependence of the robustness $R(\epsilon)$ on $\epsilon$ anticipates a trade-off between exploiting advice and worst-case robustness. We ideally seek algorithms with robustness $R(\epsilon)$ as small as possible, so that the trade-off between consistency and robustness is tight.

Our methodology for designing robust and consistent algorithms is very general, in the sense that we do not restrict to any special cases of $\cfc$ and do not consider in our analysis the explicit behavior of the advice or of any specific algorithm for $\cfc$. This is in contrast to the work of \cite{rutten_online_2022}, whose main robustness and consistency guarantees depend crucially upon the $\alpha$-polyhedral setting. Rather, we approach the task of designing robust and consistent algorithms via a more general problem of designing \emph{bicompetitive} \emph{meta-algorithms} for $\cfc$, which, informally, are ``recipes'' for combining two $\cfc$ algorithms to produce a single algorithm with competitive guarantees with respect to both input algorithms. More formally, we give the following definitions.

\begin{definition} \label{def:bicompetitive}
    An online algorithm $\alg$ for $\cfc$ is \textbf{$(c, r)$-bicompetitive} with respect to a pair of algorithms $(\algone, \algtwo)$ if $\alg$ is simultaneously $c$-competitive with respect to $\algone$ and $r$-competitive with respect to $\algtwo$. Equivalently, the cost of $\alg$ can be bounded as
    $$\cost_\alg \leq \min\left\{c\cdot\cost_{\algone}, r\cdot\cost_{\algtwo}\right\}.$$
\end{definition}

\begin{definition} \label{def:metaalgorithm}
    A \textbf{meta-algorithm} $\meta$ for $\cfc$ is a mapping $\meta : \calA_\cfc \times \calA_\cfc \to \calA_\cfc$. That is, $\meta$ takes as input two online algorithms for $\cfc$ and returns a single online algorithm for the problem. $\meta$ is said to be \textbf{$(c, r)$-bicompetitive} if its output is always $(c, r)$-bicompetitive with respect to its inputs.
\end{definition}

\sloppy It follows immediately from the previous two definitions that if $\meta$ is $(c, r)$-bicompetitive, $\adv$ is the advice, and $\rob$ is a $b$-competitive algorithm for (a subclass of) $\cfc$, then $\meta(\adv, \rob)$ is $c$-consistent and $rb$-robust. We discuss this observation in more detail in Appendix \ref{appendix:bicompetitive_robust_consistent}. Thus bicompetitive meta-algorithms give a general approach for designing robust and consistent algorithms for $\cfc$ and its subclasses.

The idea of approaching robust and consistent algorithm design via the design of bicompetitive meta-algorithms has been considered to some extent in the literature on other online problems, e.g. in the work of \cite{antoniadis_online_2020} on combining algorithms for MTS. To our knowledge, however, our specific terminology has not seen wide use in the literature.

\section{Warmup: Switching Algorithms and Their Fundamental Limits \label{section:switching}}
A natural first approach for designing bicompetitive meta-algorithms for $\cfc$ is to consider the class of switching algorithms, whose decisions switch between two other algorithms:

\begin{definition}
    A meta-algorithm $\meta$ is a \textbf{switching meta-algorithm} if, at each time $t$, the decision ${\meta_t(\algone, \algtwo)}$ made by $\meta$ resides in the set $\{\algone_t, \algtwo_t\}$.
\end{definition}

Switching algorithms have garnered significant attention in the literature on robustness and consistency in recent years, e.g., \cite{antoniadis_online_2020, lee_online_2021, angelopoulos_online_2021, rutten_online_2022}. In particular, the only robust and consistent algorithms for $\cfc$ or subclasses thereof in general dimension are the switching algorithms of \cite{antoniadis_online_2020} for MTS and \cite{rutten_online_2022} for $\acfc$. In the following proposition, we refine these prior results, showing the existence of a switching meta-algorithm for general $\cfc$ with tunable bicompetitive bound.

\begin{proposition} \label{prop:3_switching_alg}
    Suppose $\adv, \rob$ are algorithms for $\cfc$ and $\cost_\rob \geq 1$. There is a switching meta-algorithm $\switch$ (Appendix \ref{appendix:switching}, Algorithm \ref{alg:switch}) that is
    $\left(3+\calO(\epsilon), 5+\calO(\frac{1}{\epsilon^2})\right)$-bicompetitive with respect to the inputs $(\adv, \rob)$, where $\epsilon > 0$ is an algorithm hyperparameter.
\end{proposition}

Our proof of Proposition \ref{prop:3_switching_alg} follows closely that of \cite[Theorem 1]{antoniadis_online_2020}, extending it via the recent result of \cite[Theorem 5]{angelopoulos_online_2021} on linear search with a ``hint''; we present a proof in Appendix \ref{appendix:3_switching_alg_proof}. 

If $\adv$ is an advice algorithm and $\rob$ is a $C$-competitive algorithm for (a subclass of) $\cfc$, then $\switch$ yields $(3+\calO(\epsilon))$-consistency and ${(5+\calO(\frac{1}{\epsilon^2}))C}$-robustness. Notably, this does not appear to allow for \emph{arbitrary} consistency: specifically, $\switch$ cannot attain consistency less than $3$ while maintaining finite robustness. This limitation is unsurprising, since an identical lower bound holds on the algorithm for linear search on which $\switch$ is based \cite[Theorem 7]{angelopoulos_online_2021}, which can be extended to a lower bound on the bicompetitiveness of switching meta-algorithms. It is natural, then, to ask whether this lower bound also applies to the robustness and consistency of $\cfc$ algorithms. That is, can we devise a switching algorithm that, so long as it is provided with some non-adversarial, competitive algorithm $\rob$, beats 3-consistency while staying robust?

In the following theorem, which we prove in Appendix \ref{appendix:switching_lowerbound_proof}, we show that robustness and consistency also face this fundamental limit: any algorithm that switches between black-box advice and an advice-agnostic competitive algorithm cannot beat 3-consistency while preserving robustness. We prove the theorem in the $\ell^2$ setting, though the result extends to other norms such as the $\ell^\infty$ norm.

\begin{theorem} \label{theorem:switching_lowerbound}
    Consider the $\ell^2$ norm setting. Let $\adv$ be an advice algorithm, and let $\rob$ be any (deterministic) competitive algorithm for $\cbc$ that is advice-agnostic. Let $\alg$ be an online algorithm that switches between $\adv$ and $\rob$. If $\alg$ is $c$-consistent with $c < 3$, then $\alg$ cannot have finite robustness.
\end{theorem}


This lower bound implies that to obtain finite robustness alongside consistency $c < 3$ for general $\cfc$, one must venture beyond the realm of switching algorithms. This is exactly the focus of Section \ref{section:beyond_switching}, where we approach this task by exploiting the convexity of $\cfc$. First, though, we ask: are there any special cases of $\cfc$ in which switching algorithms \emph{can} obtain $(1 + \epsilon)$-consistency and finite robustness for any $\epsilon > 0$? The answer is affirmative for $\acfc$ (\cite{rutten_online_2022}), and as we show in the next proposition, such an algorithm also exists for $\ncbc$. Specifically, we propose an algorithm, $\nestedswitch$ (Algorithm \ref{alg:nested_metaalg}), which can achieve a $(1+\epsilon)$-consistent, $\calO(\frac{d}{\epsilon})$-robust trade-off for $\ncbc$ by using a simple threshold-based rule for switching between the advice and the Steiner point algorithm of \cite{bubeck_chasing_2019}. 
We prove Proposition \ref{proposition:switching_steiner} in Appendix \ref{appendix:switching_steiner_proof}.

\begin{proposition} \label{proposition:switching_steiner}
    Consider the problem of $\ncbc$ on $(\R^d, \|\cdot\|_{\ell^2})$, where the initial body $K_1$ resides in some ball $B(\yvec, r)$ of radius $r$ containing $\xvec_0$, and $\cost_\opt \geq 1$. If $\rob$ is the Steiner point algorithm (\cite{bubeck_chasing_2019}) that chooses the Steiner point of $K_t$ at each time $t$, then $\nestedswitch$ (Algorithm \ref{alg:nested_metaalg}) is $(1+\epsilon)$-consistent and $\left(1 + \frac{1}{\epsilon}\right)r(d+2)$-robust, where $\epsilon > 0$ is a hyperparameter.
\end{proposition}

\begin{algorithm}[t]
\label{alg:nested_metaalg}
\DontPrintSemicolon
\KwIn{Algorithms $\adv, \rob \in \calA_\ncbc$; hyperparameters $\epsilon, r > 0$}
\KwOut{Decisions $\xvec_1 \in K_1, \ldots, \xvec_T \in K_T$ chosen online}
\For{$t = 1, 2, \dots, T$}{
    Observe $K_t$, $\tilde{\xvec}_t \coloneqq \adv_t$, and $\svec_t \coloneqq \rob_t$\;
    \uIf{$\epsilon \cdot \cost_{\adv}(1, t) \geq r(d + 2)$}{
        $\xvec_t \leftarrow \svec_t$\; \label{eq:alg_line_1}
    }
    \uElse{
        $\xvec_t \leftarrow \tilde{\xvec}_t$\;
    }
}
\caption{$\nestedswitch(\adv, \rob; \epsilon, r)$}
\end{algorithm}

\section{Beyond Switching Algorithms: Exploiting Convexity to Break 3-Consistency \label{section:beyond_switching}}

In this section, we present our main results: two novel meta-algorithms that transcend the limitations of switching algorithms by exploiting the convexity of $\cfc$. The key insight that enables this improved performance is that \emph{hedging} between $\adv$ and $\rob$, i.e., choosing a decision that is a convex combination of the two, allows for more nuanced algorithm behavior than switching permits.

\subsection{$\interp$: a $(\sqrt{2}+\epsilon, \calO(\epsilon^{-2}))$-Bicompetitive Meta-Algorithm \label{section:main_bicompetitive_algorithm}}

We preface our first algorithm with some definitions from the geometry of real normed vector spaces $\calX = (X, \|\cdot\|)$ which we employ in the algorithm's statement and performance bound. We present abridged introductions of these notions here, giving more detail in Appendix \ref{appendix:nvs_geometry}.

We begin by introducing the \emph{rectangular constant} $\mu(\calX)$ of a normed vector space $\calX$, which is bounded between $\sqrt{2}$ and $3$, with $\mu(\calX) = \sqrt{2}$ when $\calX$ is Hilbert and $\mu(\ell^p) < 3$ for any $p \in (1, \infty)$ (\cite{joly_caracterisations_1969, baronti_revisiting_2021}). Next, we define the \emph{radial retraction}.

\begin{definition}[\cite{rieffel_lipschitz_2006}]
    On a normed vector space $\calX = (X, \|\cdot\|)$, the \textbf{radial retraction}  $\rho(\cdot\,; r) : X \to B(\bv{0}, r)$ is the metric projection onto the closed ball of radius $r \geq 0$:
    $$\rho(\xvec; r) = \begin{cases} \xvec & \text{if $\|\xvec\| \leq r$} \\ r\frac{\xvec}{\|\xvec\|} & \text{if $\|\xvec\| > r$.} \end{cases}$$
    
\end{definition}

On a fixed normed space $\calX$, the collection of radial retractions $\rho(\cdot\,; r)$ with $r > 0$ share a Lipschitz constant, which we call $k(\calX)$. It is known that $1 \leq k(\calX) \leq 2$ (\cite{thele_results_1974}), and moreover $k(\calX) \leq \mu(\calX)$ (Appendix \ref{appendix:nvs_geometry}, Proposition \ref{prop:lipschitz_rectangular_relation}).

With these definitions at our disposal, we now proceed to the main result of this section. We propose Algorithm \ref{alg:interp}, a meta-algorithm $\interp$ that takes as input two algorithms $\adv, \rob$ for (a subclass of) $\cfc$, and hyperparameters $\epsilon, \gamma, \delta > 0$ satisfying $2\gamma + 2\delta = \epsilon$. $\interp$ works as follows: at each time $t$, if the cost of $\rob$ so far is a substantial fraction of the cost of $\adv$, then $\interp$ can move to $\adv_t$ while staying competitive with respect to $\rob$, and it does so (line \ref{algline:advice_general}). Otherwise, $\interp$ moves to a point $\xvec_t$ determined by the series of radial projections (lines \ref{algline:firstargmax}, \ref{algline:z_t-y_t_general}, and \ref{algline:notadvice_general}), which intuitively guide $\interp$ to take a ``greedy step'' toward the decision made by $\rob$ while still remaining close enough to $\adv$ so as to maintain a consistency guarantee.

\begin{algorithm}
\label{alg:interp}
\DontPrintSemicolon
\KwIn{Algorithms $\adv, \rob$; hyperparameters $\epsilon > 0$ and $\gamma > 0, \delta > 0$ satisfying $2\gamma + 2\delta = \epsilon$}
\KwOut{Decisions $\xvec_1, \ldots, \xvec_T$ chosen online}
\For{$t = 1, 2, \dots, T$}{
    Observe $f_t$, $\tilde{\xvec}_t \coloneqq \adv_t$, and $\svec_t \coloneqq \rob_t$\;
    \uIf{$\cost_{\rob}(1, t) \geq \delta \cdot \cost_{\adv}(1, t)$}{
        $\xvec_t \leftarrow \tilde{\xvec}_t$\; \label{algline:advice_general}
    }
    \uElse{$\yvec_t \leftarrow \svec_{t-1} + \rho\left(\tilde{\xvec}_t - \svec_{t-1}; \|\xvec_{t-1} - \svec_{t-1}\|\right)$ \label{algline:firstargmax}\; 
    $\zvec_t \leftarrow \svec_{t-1} + \rho\left(\yvec_t - \svec_{t-1}; \max\{\|\yvec_t - \svec_{t-1}\| - \gamma\cdot \cost_{\adv}(t, t), 0\}\right)$\; \label{algline:z_t-y_t_general}
    $\xvec_t \leftarrow \svec_t + \rho(\tilde{\xvec}_t - \svec_t; \|\zvec_t - \svec_{t-1}\|)$\; \label{algline:notadvice_general}
    }
}
\caption{$\interp(\adv, \rob; \epsilon, \gamma, \delta)$}
\end{algorithm}

We characterize the bicompetitive performance of $\interp$ in the following theorem, which holds in any normed vector space $\calX$ of arbitrary dimension.

\begin{theorem} \label{theorem:consistent_robust_general}
     $\interp$ (Algorithm \ref{alg:interp}) is 
    $$\left(\mu(\calX)+\epsilon, 1 + \frac{k(\calX)}{\gamma} + \frac{\mu(\calX) + \epsilon + 1 + \frac{k(\calX)}{\gamma}}{\delta}\right)\text{--bicompetitive}$$
    with respect to $(\adv, \rob)$. With $\gamma, \delta$ chosen optimally, the bound is $(\mu(\calX) + \epsilon, \calO(\epsilon^{-2}))$. 
    
    In particular, if $\adv$ is advice and $\rob$ is $C$-competitive for (a subclass of) $\cfc$, then $\interp$ is $(\mu(\calX)+\epsilon)$-consistent and $\calO(C\epsilon^{-2})$-robust.
\end{theorem}

Notably, $\interp$ (Algorithm \ref{alg:interp}) strictly improves on the 3-consistent lower bound for switching meta-algorithms in any $\ell^p$ space with $1 < p < \infty$, in which it holds that $\mu(\ell^p) < 3$. Moreover, it obtains consistency $(\sqrt{2} + \epsilon)$ in any Hilbert space. We prove Theorem \ref{theorem:consistent_robust_general} and give details regarding optimal selection of the parameters $\gamma, \delta$ in Appendix \ref{appendix:main_general_theorem_proof}. The proof employs two potential function arguments with different potential functions for the bounds with respect to $\adv$ and $\rob$. Moreover, the generality of the theorem's setting requires the development of several geometric results characterizing the radial projection and its relation to the rectangular constant in arbitrary-dimensional normed vector spaces, which we present in Appendix \ref{appendix:geometric_lemmas} prior to the main proof. These results are crucial for enabling $\interp$'s robustness and consistency in the general setting and elucidate the presence of the constants $\mu(\calX)$ and $k(\calX)$ in its bicompetitive bound. 

We also detail robustness and consistency corollaries of Theorem  \ref{theorem:consistent_robust_general} for multiple subclasses of $\cfc$ in Appendix \ref{appendix:laundry_list}. In particular, Theorem \ref{theorem:consistent_robust_general} and Table \ref{table:sota} imply an algorithm for $\cfc$ and $\cbc$ on $\R^d$ with any norm that is ${(\mu(\R^d, \|\cdot\|)+\epsilon)}$-consistent and $\calO(\frac{d}{\epsilon^2})$-robust; we also obtain an algorithm for $\acfc$ that is $(\mu(\calX)+\epsilon)$-consistent and $\calO(\frac{1}{\alpha\epsilon^2})$-robust for $\acfc$ on any normed vector space $\calX$.

\subsection{Attaining $(1+\epsilon)$-Consistency in Bounded Instances with $\binterp$}
In the preceding section, we proved that in the Hilbert space setting, $\interp$ (Algorithm \ref{alg:interp}) obtains consistency $(\sqrt{2} + \epsilon)$ while remaining competitive with respect to $\rob$. While this is a significant improvement on the limit of 3-consistency faced by switching algorithms, the question remains: can we devise an algorithm that achieves $(1+\epsilon)$-consistency and $R(\epsilon) < \infty$ competitiveness with respect to $\rob$ for \emph{any} $\epsilon > 0$, in \emph{any} normed vector space? In this section, we provide a simple sufficient condition under which this is possible: if there exists some constant $D \in \R_+$ for which $\|\adv_t - \rob_t\| \leq D$ for all $t$, then there is a meta-algorithm that is $(1+\epsilon, \calO(\frac{D}{\epsilon}))$-bicompetitive with respect to $(\adv, \rob)$. We call this condition \emph{$D$-boundedness} of $\adv$ and $\rob$; it arises naturally in a number of settings, for example in any $\cbc$ instance in which the diameter of each body $K_t$ is bounded by $D$.

We present the algorithm achieving this bicompetitive bound, $\binterp$, in Algorithm \ref{alg:binterp}. Just like $\interp$, $\binterp$ takes as input two algorithms $\adv, \rob$ for (a subclass of) $\cfc$, and hyperparameters $\epsilon, \gamma, \delta > 0$ satisfying $2\gamma + 2\delta = \epsilon$. At a high level, $\binterp$ operates similarly to $\interp$, though it takes smaller greedy steps toward $\rob$, enabling it to maintain $(1+\epsilon)$-consistency. Specifically, $\binterp$ works as follows: if the cost of $\rob$ is a sufficient fraction of the cost of $\adv$, then $\binterp$ moves to $\adv_t$ (line \ref{bd_algline:advice_general}). Otherwise, it selects an auxiliary point $\yvec_t$ as the point along the segment $[\svec_t, \tilde{\xvec}_t]$ with the same relative position as $\xvec_{t-1}$ on the segment $[\svec_{t-1}, \tilde{\xvec}_{t-1}]$ (line \ref{bd_algline:y_t}), and then chooses $\xvec_t$ by taking a greedy step toward $\svec_t$ from $\yvec_t$ (line \ref{bd_algline:notadvice_general}).

\begin{algorithm}
\label{alg:binterp}
\DontPrintSemicolon
\KwIn{Algorithms $\adv, \rob$; hyperparameters $\epsilon > 0$ and $\gamma > 0, \delta > 0$ satisfying $2\gamma + 2\delta = \epsilon$}
\KwOut{Decisions $\xvec_1, \ldots, \xvec_T$ chosen online}
\For{$t = 1, 2, \dots, T$}{
    Observe $f_t$, $\tilde{\xvec}_t \coloneqq \adv_t$, and $\svec_t \coloneqq \rob_t$\;
    \uIf{$\cost_{\rob}(1, t) \geq \delta \cdot \cost_{\adv}(1, t)$}{
        $\xvec_t \leftarrow \tilde{\xvec}_t$\; \label{bd_algline:advice_general}
    }
    \uElse{
        $\nu \leftarrow \frac{\|\xvec_{t-1} - \svec_{t-1}\|}{\|\tilde{\xvec}_{t-1} - \svec_{t-1}\|}$ if $\tilde{\xvec}_{t-1} \neq \svec_{t-1}$, otherwise $\nu \leftarrow 0$\; \label{bd_algline:nu}
        $\yvec_t \leftarrow \nu \tilde{\xvec}_t + (1-\nu)\svec_t$\; \label{bd_algline:y_t}
        $\xvec_t \leftarrow \svec_t + \rho\left(\yvec_t - \svec_t; \max\{\|\yvec_t - \svec_t\| - \gamma\cdot \cost_{\adv}(t, t), 0\}\right)$\; \label{bd_algline:notadvice_general}
    }
}
\caption{$\binterp(\adv, \rob; \epsilon, \gamma, \delta)$}
\end{algorithm}

We present the performance result for $\binterp$ in Theorem \ref{theorem:bounded_bicompetitive}; like Theorem \ref{theorem:consistent_robust_general}, the result holds in any normed vector space $\calX$ of arbitrary dimension.

\begin{theorem} \label{theorem:bounded_bicompetitive}
    Suppose that $\adv$ and $\rob$ are $D$-bounded, i.e., $\|\adv_t - \rob_t\| \leq D$ for all $t \in [T]$; and assume that $\cost_\rob \geq 1$. Then $\binterp$ (Algorithm \ref{alg:binterp}) is 
    $$\left(1+\epsilon, D + \frac{D}{\gamma} + \frac{1+\epsilon}{\delta}\right)\text{--bicompetitive}$$
    with respect to $(\adv, \rob)$. With $\gamma, \delta$ chosen optimally, the bound is ${(1+\epsilon, \calO(\frac{D}{\epsilon}))}$. 
    
    In particular, if $\adv$ is advice and $\rob$ is $C$-competitive for (a subclass of) $\cfc$, then $\binterp$ is $(1+\epsilon)$-consistent and $\calO(\frac{CD}{\epsilon})$-robust.
\end{theorem}

Remarkably, Theorem \ref{theorem:bounded_bicompetitive} states that $\binterp$ not only improves on the consistency of $\interp$, but it also strictly improves upon $\interp$'s $\calO(\epsilon^{-2})$ competitiveness with respect to $\rob$ when $D = o(\epsilon^{-1})$. Moreover, $\binterp$ substantially improves on the randomized switching algorithm of \cite{antoniadis_online_2020} in the $D$-bounded setting, providing deterministic and tunable robustness and consistency guarantees with no additive factor in the consistency term. We give a proof of Theorem \ref{theorem:bounded_bicompetitive}, as well as details on optimal parameter selection, in Appendix \ref{appendix:bounded_bicompetitive}. The argument follows a similar line of reasoning as that of our proof of Theorem \ref{theorem:consistent_robust_general}, though in the proof of competitiveness with respect to $\rob$ (i.e., robustness), we employ a novel potential function constructed via the ratio between the respective distances of $\xvec_t$ and $\tilde{\xvec}_t$ to $\svec_t$.

We detail robustness and consistency corollaries of Theorem \ref{theorem:bounded_bicompetitive} for multiple subclasses of $\cfc$ in Appendix \ref{appendix:laundry_list}. In particular, Theorem \ref{theorem:bounded_bicompetitive} and Table \ref{table:sota} imply an algorithm for $\cfc$ and $\cbc$ with any norm that is $(1+\epsilon)$-consistent and $\calO(\frac{dD}{\epsilon})$-robust on $D$-bounded instances. We also obtain an algorithm for $\acfc$ in the $D$-bounded finite-dimensional Euclidean setting that achieves $(1+\epsilon)$-consistency and $\calO(\frac{D}{\alpha^{1/2}\epsilon})$-robustness. This latter bound is nearly tight for $\acfc$: \citet[Theorem 3.6]{rutten_online_2022} prove a lower bound of $\calO(\frac{1}{\alpha})$ robustness for any $(1+\alpha)$-consistent $\acfc$ algorithm. Choosing $\epsilon = \alpha$, Theorem \ref{theorem:bounded_bicompetitive} gives us $(1+\alpha)$-consistency and $\calO(\frac{1}{\alpha^{3/2}})$-robustness when $D = O(1)$ in $\alpha$, leaving a gap of just $\calO(\alpha^{-1/2})$ between the upper and lower bounds. We leave to future work the question of whether the upper and lower bounds can be made tight and whether the factor of $D$ (and more generally the $D$-boundedness assumption) can be dropped.

\section{Conclusion}
In this work, we examine the question of integrating black-box advice into algorithms for convex function chasing using the notions of robustness and consistency from the literature on online algorithms with machine-learned advice. We first propose an algorithm that switches between the decisions of an arbitrary $C$-competitive algorithm $\rob$ and the advice, showing that it obtains $(3+\calO(\epsilon))$-consistency and finite robustness for any $\epsilon > 0$.  We moreover show that this is optimal, in the sense that \emph{no} switching algorithm can improve upon 3-consistency while maintaining finite robustness. We then move beyond switching algorithms, and propose two algorithms, $\interp$ and $\binterp$, which obtain improved robustness and consistency guarantees by exploiting the convexity inherent in the $\cfc$ problem. In particular, $\interp$ obtains $(\sqrt{2}+\epsilon)$-consistency and $\calO(\frac{C}{\epsilon^2})$-robustness, and under the additional assumption of $D$-boundedness, $\binterp$ can obtain $(1+\epsilon)$-consistency and $\calO(\frac{CD}{\epsilon})$-robustness. We show that $\binterp$ is nearly optimal for the problem of $\cfc$ with $\alpha$-polyhedral hitting cost functions $f_t$, so long as $D = O(1)$ in $\alpha$.

Several interesting questions remain open for future work: in particular, (a) the question of whether $(1+\epsilon)$-consistency and finite robustness can be obtained for general $\cfc$ without the $D$-boundedness assumption, and (b) the question of tight lower bounds on robustness and consistency for $\cfc$ and its many subclasses. Specifically, for the case of $\cfc$ in general, we pose the question of whether $(1+\epsilon)$-consistency is possible together with $\calO(\frac{d}{\epsilon})$-robustness, or even whether the dependence on $\epsilon$ and $d$ can be further improved in the robustness bound.

\acks{The authors thank Eitan Levin for several helpful discussions. The authors acknowledge support from an NSF Graduate Research Fellowship (DGE-1745301), NSF grants CNS-2146814, CPS-2136197, CNS-2106403, and NGSDI-2105648, and Amazon AWS.}

\bibliography{main}

\appendix


\section{Preliminaries \label{appendix:prelim}}
In part \ref{appendix:cfc_subclasses} of this appendix, we provide more detailed definitions of the subclasses of convex function chasing considered in this work, including both those introduced in Section \ref{section:cfc_subclasses} as well as several additional special cases which we refer to in the robustness and consistency results given in Appendix \ref{appendix:laundry_list}. We also review state-of-the-art competitive algorithms for each subclass, which we summarize in Table \ref{table:sota_extended}, which is an extended version of Table \ref{table:sota} in the main text. Then, in part \ref{appendix:bicompetitive_robust_consistent}, we elaborate on the claim made in Section \ref{section:robust_consistent_intro} that a $(c, r)$-bicompetitive meta-algorithm for $\cfc$, along with a $b$-competitive algorithm for a subclass of $\cfc$, together yield a $c$-consistent and $rb$-robust algorithm for that subclass.

\begin{table}
\begin{center}
\renewcommand*{\arraystretch}{1.2}
\begin{tabular}{|c||c|c|} 
  \hline
  Problem & State-of-the-Art Competitive Ratio & Setting \\
  \hline\hline
  $\cfc$ & $d+1$ & $\R^d$ with any norm \\
  \hline 
  $\cbc$ & $d$ & $\R^d$ with any norm \\
  \hline 
  $\kcbc$ & $2k+1$ & $\R^d$ with $\ell^2$ norm \\
  \hline 
  $\acfc$ & $\max\left\{1, \frac{2}{\alpha}\right\}$ & Any normed vector space \\
  \hline
  $\acfc$ & $\calO\left(\frac{1}{\alpha^{1/2}}\right)$ & $\R^d$ with $\ell^2$ norm \\
  \hline
  $\kgcfc$ & $(2 + 2\sqrt{2})2^{\gamma/2}\kappa$ & $\R^d$ with $\ell^2$ norm \\
  \hline
\end{tabular}
\caption{Competitive ratios for state-of-the-art algorithms on various subclasses of $\cfc$.}
\label{table:sota_extended}
\end{center}
\end{table}

\subsection{Subclasses of $\cfc$ \label{appendix:cfc_subclasses}}
\subsubsection{Convex body chasing \label{appendix:cbc_cfc_equiv}} 
In the problem of convex body chasing ($\cbc$) on a normed vector space $(X, \|\cdot\|)$, at each time $t$ a decision-maker is given a convex body $K_t \subseteq X$ and faces the requirement that their decision $\xvec_t$ must reside within $K_t$. A further special case of $\cbc$ is the problem of \emph{nested} convex body chasing ($\ncbc$), in which subsequent bodies are nested, i.e. $K_t \supseteq K_{t+1}$ for each $t$. We define the set of all online algorithms which are \emph{feasible} for $\cbc$, i.e. which produce decisions residing within the convex body $K_t$ at each time, as $\calA_\cbc$. We define $\calA_\ncbc$ similarly as the set of all online algorithms which are feasible for $\ncbc$. \cite{sellke_chasing_2020} proved that an algorithm based on a functional generalization of the Steiner point of a convex body achieves competitive ratio $d$ for $\cbc$ and $d + 1$ for general $\cfc$ in $\R^d$ equipped with any norm.

The problem of convex body chasing can easily be seen as a special case of $\cfc$ in which each hitting cost $f_t$ is the $\{0, \infty\}$ indicator of the convex set $K_t$. That is,
$$f_t(\xvec) = \begin{cases} 0 & \text{if $\xvec \in K_t$} \\ \infty &\text{otherwise.}\end{cases}$$
As noted in \cite{sellke_chasing_2020}, we need not even require hitting costs to take infinite values to recover convex body chasing from function chasing. Indeed, restricting to the finite-dimensional setting,\footnote{This restriction is natural because no algorithm can be competitive for $\cbc$ in the infinite-dimensional setting. We impose the restriction in order to ensure existence of a metric projection onto $K_t$.} consider $f_t$ defined as
$$f_t(\xvec) = 3\cdot d(\xvec, K_t) = 3\min_{\yvec \in K_t} \|\xvec - \yvec\|.$$
Then any algorithm $\alg \in \calA_\cfc$ yields a set of decisions $\alg_1, \ldots, \alg_T$ on the instance $(\xvec_0, f_1, \ldots, f_T)$; and moreover, $\alg$ can be transformed into an algorithm $\alg' \in \calA_\cfc$ with strictly improved cost, and which in particular incurs no hitting cost, by setting
$$\alg_t' = \begin{cases}\alg_t & \text{if $\alg_t \in K_t$} \\ \Pi_{K_t}\alg_t &\text{otherwise.} \end{cases}$$
Clearly each decision $\alg'_t$ resides in the convex body $K_t$; thus $\alg'$ is a feasible online algorithm for $\cbc$, i.e., $\alg' \in \calA_\cbc$. Moreover, the cost of $\alg'$ on the $\cbc$ instance is identical to its cost for the corresponding $\cfc$ instance, since it incurs no hitting cost. It follows that the competitive ratio of $\alg'$ for the $\cbc$ problem is at most the competitive ratio of $\alg$ as a $\cfc$ algorithm, since $\opt$ for a $\cbc$ instance and its corresponding $\cfc$ instance always coincide. In short, a $C$-competitive algorithm for $\cfc$ is also $C$-competitive for $\cbc$. 

\begin{remark}
    The preceding line of reasoning can be extended to show that a $C$-competitive algorithm for $\cfc$ gives a $C$-competitive algorithm for $\cfc$ with strict decision constraints, i.e., where at time $t$, the decision $\xvec_t$ must both reside in some convex body $K_t$ and also incurs a convex hitting cost $f_t(\xvec_t)$.
\end{remark}

\subsubsection{Chasing low-dimensional convex bodies}
A special case of convex body chasing that has received significant attention is the problem of chasing low-dimensional bodies in higher-dimensional space: indeed, the seminal work of \cite{friedman_convex_1993} began by addressing the problem of chasing lines in the Euclidean plane. \cite{bienkowski_better_2019} later presented a 3-competitive algorithm for chasing lines in $(\R^d, \|\cdot\|_{\ell^2})$, and most recently \cite{argue_dimension-free_2020} gave an algorithm that is $(2k+1)$-competitive for chasing convex bodies lying in $k$-dimensional affine subspaces, regardless of the dimension $d$ of the underlying Euclidean space. Motivated by this last result, we define the problem of $k$-dimensional convex body chasing ($\kcbc$), comprised of all instances of $\cbc$ in which each body $K_t$ lies within an affine subspace of dimension at most $k$ -- i.e., $\dim\aff K_t \leq k$ for all $t$.

\subsubsection{$\alpha$-polyhedral convex function chasing}
A class of functions that has been studied extensively in the literature on online optimization with switching costs ($\soco$) is the class of globally $\alpha$-polyhedral functions, e.g., \cite{chen_smoothed_2018, zhang_revisiting_2021}, which are defined as follows.

\begin{definition}
    Let $(X, \|\cdot\|)$ be a normed vector space, and let $\alpha > 0$. A function $f : X \to \R_+$ is \textbf{globally $\alpha$-polyhedral} if it has unique minimizer $\xvec^* \in X$, and in addition,
    $$f(\xvec) \geq f(\xvec^*) +  \alpha\|\xvec - \xvec^*\| \qquad\text{for all $\xvec \in X$}.$$
\end{definition}

Roughly speaking, a globally $\alpha$-polyhedral function has a unique minimizer, away from which it grows with slope at least $\alpha$. For a fixed $\alpha > 0$, we define the problem of \emph{$\alpha$-polyhedral convex function chasing} ($\acfc$) comprised of all those problem instances of $\cfc$ in which, in addition to being convex, each function $f_t$ is also globally $\alpha$-polyhedral. $\acfc$ has been widely studied due to its admitting algorithms with ``dimension-free'' competitive ratios: \cite{zhang_revisiting_2021} showed that a greedy algorithm that simply moves to the minimizer $\xvec_t^*$ of each function $f_t$ is $\max\{1, \frac{2}{\alpha}\}$-competitive for $\acfc$ in any normed vector space. In the setting of $\R^d$ with the $\ell^2$ norm, \cite{lin_personal_2022} gave an algorithm augmenting the Online Balanced Descent algorithm of \cite{chen_smoothed_2018} to achieve a competitive ratio of $\calO(\frac{1}{\alpha^{1/2}})$.

\subsubsection{$(\kappa, \gamma)$-well-centered convex function chasing}
Another class of functions that has received attention in the design of algorithms for subclasses of $\cfc$ with dimension-free competitive ratios is the set of $(\kappa,\gamma)$-well-centered functions, introduced by \cite{argue_dimension-free_2020}:

\begin{definition}
    Let $(X, \|\cdot\|)$ be a normed vector space, and let $\kappa, \gamma \geq 1$. A function $f : \R^d \to \R_+$ with minimizer $\xvec^*$ is \textbf{$(\kappa, \gamma)$-well-centered} if there exists some $a > 0$ such that
    $$\frac{a}{2}\|\xvec - \xvec^*\|^\gamma \leq f(\xvec) \leq \frac{a\kappa}{2}\|\xvec - \xvec^*\|^\gamma \qquad \text{for all $\xvec \in X$}.$$
\end{definition}

Intuitively, the growth rate of a $(\kappa, \gamma)$-well-centered function away from its minimizer (as measured with the ``distance'' $\|\cdot\|^\gamma$) is bounded above and below, and the ratio of these bounds is at most $\kappa$. For fixed $\kappa, \gamma \geq 1$, we define the problem of \emph{$(\kappa, \gamma)$-well-centered convex function chasing} ($\kgcfc$) comprised of all those problem instances of $\cfc$ in which each $f_t$ is $(\kappa, \gamma)$-well-centered. \cite{argue_dimension-free_2020} showed that the "Move towards Minimizer" algorithm is $(2 + 2\sqrt{2})2^{\gamma/2}\kappa$-competitive for $\kgcfc$ on $\R^d$ equipped with the $\ell^2$ norm.

\subsection{Bicompetitive meta-algorithms give robust and consistent algorithms \label{appendix:bicompetitive_robust_consistent}}

In this section, we briefly justify the claim that if $\meta$ is a $(c, r)$-bicompetitive meta-algorithm for $\cfc$, $\adv$ is an advice algorithm, and $\rob$ is a $b$-competitive online algorithm for a subclass of $\cfc$, then $\meta(\adv, \rob)$ is $c$-consistent and $rb$-robust for that subclass. This is straightforward to see for non-$\cbc$ subclasses, or more generally, for any subclass of $\cfc$ which does not involve hard constraints on the decisions $\xvec_t$. In particular, $\rob$ being $b$-competitive means that $\cost_\rob \leq b \cdot \cost_\opt$, and so $(c, r)$-bicompetitiveness of $\meta$ implies that both $\cost_{\meta(\adv, \rob)} \leq c \cdot \cost_\adv$ and $\cost_{\meta(\adv, \rob)} \leq r \cdot \cost_\rob \leq rb \cdot \cost_\opt$, as desired. 

The only subclasses that require more careful justification are those, such as $\cbc$, with hard constraints on the decisions. However, so long as the advice always gives feasible decisions -- e.g., in the $\cbc$ case, $\adv \in \calA_\cbc$, so $\adv_t \in K_t$ for each $t$ -- then we can obtain the same result by applying the reasoning from Appendix \ref{appendix:cbc_cfc_equiv} on equivalent $\cfc$ reformulations of instances with hard constraints. That is, on any instance of the subclass, we must simply run $\meta(\adv, \rob)$ on its equivalent reformulation as a $\cfc$ instance, and we thereby obtain the same guarantees of $c$-consistency and $rb$-robustness for the subclass.

\section{Switching algorithms \label{appendix:switching}}
\subsection{The meta-algorithm $\switch$ \label{appendix:switch_algorithm}}
We give the meta-algorithm $\switch$, which takes as hyperparameters $b > 1$ and $\delta \in (0, 1]$, in Algorithm \ref{alg:switch}. In order to reduce the two hyperparameters $b, \delta$ to a single hyperparameter $\epsilon$ as in the statement of Proposition \ref{prop:3_switching_alg}, we simply introduce an auxiliary variable $\gamma$ and make the substitutions $\delta \leftarrow b\gamma^2 - b^{-1}$, $b \leftarrow \sqrt{\gamma^{-2} + 1}$, and $\gamma \leftarrow \sqrt{\frac{\epsilon}{4}}$.

\begin{algorithm}
\label{alg:switch}
\DontPrintSemicolon
\KwIn{Algorithms $\adv, \rob \in \calA_\cfc$; hyperparameters $b > 1$, $\delta \in (0, 1]$}
\KwOut{Decisions $\xvec_1 , \ldots, \xvec_T$ chosen online}
$i \leftarrow 0$ \;
\While{problem instance has not ended}{
    \uIf{$i \equiv 0 \mod 2$}{
        $\xvec_t \leftarrow \adv_t$ until the last time $t$ that $\cost_{\adv}(1, t) \leq b^i$\; \label{switch:adv}
        $i \leftarrow i + 1$\;
    }
    \uElse{
        $\xvec_t \leftarrow \rob_t$ until the last time $t$ that $\cost_{\rob}(1, t) \leq \delta b^i$\;
        $i \leftarrow i + 1$\;
    }
}
\caption{$\switch(\adv, \rob; b, \delta)$}
\end{algorithm}


\subsection{Proof of Proposition \ref{prop:3_switching_alg} \label{appendix:3_switching_alg_proof}}

The proof follows the argument of \cite[Theorem 5]{angelopoulos_online_2021} and uses a similar line of reasoning as \cite[Theorems 1, 18]{antoniadis_online_2020} in applying an algorithm for linear search to $\cfc$.

Each value of $i$ encountered in the execution of Algorithm \ref{alg:switch} is taken to refer to a \emph{phase} of the algorithm; every decision $\xvec_t$ made during a particular value of $i$ is said to take place during the $i$\textsuperscript{th} phase. Our strategy will be to bound the cost that the algorithm incurs in each phase $i$, including the cost it takes to switch from the last decision of the previous phase $i - 1$. 

As a base case, consider $i = 0$. The total cost incurred by the algorithm during this phase is bounded by $b^i = b^0 = 1$, by line \ref{switch:adv} of the algorithm.

Now consider phase $i > 0$, and assume that $i$ is odd; after proving the cost bound for phase $i$ in the odd case, we will state the corresponding bound for the even case, which follows a nearly identical argument. Let $\underline{t}$ be the last timestep in the $(i-1)$\textsuperscript{th} phase -- that is, $\underline{t}$ is defined such that $\cost_\adv(1, \underline{t}+1) > b^{i-1}$. We will assume that $\cost_\adv(1, \underline{t}) \leq b^{i-1}$, i.e., the algorithm makes at least one decision during phase $(i-1)$, selecting $\xvec_{\underline{t}} = \adv_{\underline{t}}$; but the upper bound we obtain will also apply to the case where the $(i-1)$\textsuperscript{th} phase is vacuous. Let $\overline{t}$ be the last timestep corresponding to phase $i$, i.e., $\overline{t} \geq \underline{t}$ is defined such that $\cost_\rob(1, \overline{t}+1) > \delta b^i$. If $\overline{t} = \underline{t}$, then clearly no decisions are made during phase $i$, so no cost is incurred during this phase. On the other hand, if $\overline{t} > \underline{t}$, then certainly $\cost_\rob(1, \overline{t}) \leq \delta b^i$, so the cost incurred by $\switch$ during phase $i$, starting from its position at time $\underline{t}$, can be bounded as
\begin{align*}
    \cost_\switch(\underline{t}+1, \overline{t}) &= \sum_{t = \underline{t}+1}^{\overline{t}} f_t(\xvec_t) + \|\xvec_t - \xvec_{t-1}\| \\
    &= f_{\underline{t}+1}(\xvec_{\underline{t}+1}) + \|\xvec_{\underline{t}+1} - \xvec_{\underline{t}}\| + \sum_{t = \underline{t}+2}^{\overline{t}} f_t(\xvec_t) + \|\xvec_t - \xvec_{t-1}\| \\
    &\leq f_{\underline{t}+1}(\xvec_{\underline{t}+1}) + \|\xvec_{\underline{t}+1} - \xvec_0\| + \|\xvec_{\underline{t}} - \xvec_0\| + \sum_{t = \underline{t}+2}^{\overline{t}} f_t(\xvec_t) + \|\xvec_t - \xvec_{t-1}\| \\
    &= f_{\underline{t}+1}(\rob_{\underline{t}+1}) + \|\rob_{\underline{t}+1} - \xvec_0\| + \|\adv_{\underline{t}} - \xvec_0\|\\
    &\qquad + \sum_{t = \underline{t}+2}^{\overline{t}} f_t(\rob_t) + \|\rob_t - \rob_{t-1}\| \\
    &\leq \cost_\rob(1, \overline{t}) + \cost_\adv(1, \underline{t}) \\
    &\leq \delta b^i + b^{i-1}.
\end{align*}
where the first two bounds use the triangle inequality, and the last bound follows by construction of $\overline{t}$ and $\underline{t}$. By a very similar argument, if $i$ is even, we can bound the cost incurred by $\switch$ during phase $i$ as $\delta b^{i-1} + b^i$. Then the total cost expenditure of $\switch$ through the end of some phase $N > 0$ is at most 
\begin{align} 
    &b^0 + \sum_{i = 0}^{\lfloor\frac{N-1}{2}\rfloor} (\delta b^{2i + 1} + b^{2i}) + \sum_{j = 1}^{\lfloor\frac{N}{2}\rfloor} (\delta b^{2j - 1} + b^{2j}) \nonumber\\
    =\quad& \begin{cases} b^N + 2\sum_{i=0}^{\frac{N}{2} - 1} (b^{2i} + \delta b^{2i + 1}) & \text{if $N$ is even} \\ \delta b^N + 2b^{N-1} + 2\sum_{i=0}^{\frac{N-3}{2}} (b^{2i} + \delta b^{2i + 1}) & \text{if $N$ is odd}\end{cases} \label{eq:switch_cost_cases}
\end{align}

Suppose then that the instance ends at time $T$ \emph{during} phase $N$. We break into cases depending on the value of $N$. 

First, if $N = 0$, then clearly $\cost_\switch = \cost_\adv$, so $\switch$ is 1-competitive with respect to $\adv$. Moreover, since $\cost_\adv \leq b^0 = 1$, then by the assumption in the proposition statement that $\cost_\rob \geq 1$, it follows that $\switch$ is at most 1-competitive with respect to $\rob$.

Second, if $N = 1$, then $\cost_\switch \leq 2b^0 + \cost_\rob  = 2 + \cost_\rob$. Since $\cost_\adv > 1$ and $\cost_\rob \leq \delta b$ (due to the instance ending at phase $N = 1$), this means that $\switch$ is at most $(2 + \delta b)$-competitive with respect to $\adv$. Moreover, by assumption $\cost_\rob \geq 1$, $\switch$ is at most $3$-competitive with respect to $\rob$. 

Next, suppose $N > 1$ and $N$ is even. Then we have
$$\cost_\switch \leq 2\sum_{i=0}^{\frac{N}{2} - 1} (b^{2i} + \delta b^{2i + 1}) + \cost_\adv$$
which follows by applying (\ref{eq:switch_cost_cases}) to bound cost through phase $(N - 1)$, and bounding the remaining cost by $\delta b^{N-1} + \cost_\adv$, i.e., the cost to switch back to $\adv$ and follow it until the instance ends. Then note that $\cost_\adv \geq b^{N - 2}$ by definition of phase; introducing the substitution $2k \coloneqq N - 2$, we find that the competitive ratio of $\switch$ with respect to $\adv$ is bounded as
\begin{align*}
    \frac{\cost_\switch}{\cost_\adv} &\leq 1 + 2\frac{\sum_{i=0}^{k} (b^{2i} + \delta b^{2i + 1})}{b^{2k}} \\
    &= 1 + 2\left(\frac{b^{2k+2} - 1}{b^{2k}(b^2 - 1)} + \delta b\frac{b^{2k+2} - 1}{b^{2k}(b^2 - 1)}\right) \\
    &\leq 1 + 2\left(\frac{b^2}{b^2 - 1} + \delta \frac{b^3}{b^2 - 1}\right).
\end{align*}

On the other hand, we know that $\cost_\adv \leq b^N$ and $\cost_\rob \geq \delta b^{N-1}$, so by similar reasoning the competitive ratio of $\switch$ with respect to $\rob$ is bounded as
\begin{align*}
    \frac{\cost_\switch}{\cost_\rob} &\leq \frac{b}{\delta} + 2\frac{\sum_{i=0}^{k} (b^{2i} + \delta b^{2i + 1})}{\delta b^{2k+1}} \\
    &\leq \frac{b}{\delta} + 2\left(\frac{b}{\delta(b^2 - 1)} + \frac{b^2}{b^2 - 1}\right).
\end{align*}

Finally, consider $N > 1$ for odd $N$. Then 
$$\cost_\switch \leq 2\sum_{i=0}^{\frac{N-1}{2}} b^{2i} + \sum_{i=0}^{\frac{N-3}{2}} \delta b^{2i + 1} + \cost_\rob.$$
Noting that $\cost_\rob \leq \delta b^N$, $\cost_\adv \geq b^{N-1}$, and making the substitution $2k = N - 1$, we obtain that the competitive ratio of $\switch$ with respect to $\adv$ is bounded as
\begin{align*}
    \frac{\cost_\switch}{\cost_\adv} &\leq \delta b + 2\frac{\sum_{i=0}^{k} b^{2i} + \sum_{i=0}^{k-1} \delta b^{2i + 1}}{b^{2k}} \\
    &\leq \delta b + 2\left(\frac{b^2}{b^2 - 1} + \delta \frac{b}{b^2 - 1}\right).
\end{align*}

On the other hand, we know that $\cost_\rob \geq \delta b^{N-2} = \delta b^{2k-1}$. Thus the competitive ratio of $\switch$ with respect to $\rob$ is bounded as
\begin{align*}
    \frac{\cost_\switch}{\cost_\rob} &\leq 1 + 2\frac{\sum_{i=0}^{k} b^{2i} + \sum_{i=0}^{k-1} \delta b^{2i + 1}}{\delta b^{2k-1}} \\
    &\leq 1+ 2\left(\frac{1}{\delta}\frac{b^3}{b^2 - 1} + \frac{b^2}{b^2 - 1}\right).
\end{align*}

Combining these various cases, we obtain that $\switch$ is
$$\left(1 + 2\left(\frac{b^2}{b^2 - 1} + \delta \frac{b^3}{b^2 - 1}\right), 1+ 2\left(\frac{b^2}{b^2 - 1} + \frac{1}{\delta}\frac{b^3}{b^2 - 1}\right)\right)\text{--bicompetitive}$$
 with respect to $(\adv, \rob)$. Introducing an auxiliary parameter $\gamma$ and making the substitutions $\delta \leftarrow b\gamma^2 - b^{-1}$, $b \leftarrow \sqrt{\gamma^{-2} + 1}$, and $\gamma \leftarrow \sqrt{\frac{\epsilon}{4}}$, we arrive at the bicompetitive bound in terms of $\epsilon$ stated in the proposition. \jmlrQED

\subsection{Proof of Theorem \ref{theorem:switching_lowerbound} \label{appendix:switching_lowerbound_proof}}

We consider the setting of $\R^d$ with the $\ell^2$ norm, where the advice $\adv$ is adversarial and $\rob$ is an \emph{arbitrary} $b$-competitive algorithm for $\cbc$, with $b < \infty$; $\alg$ is any algorithm that switches between $\rob$ and $\adv$. For simplicity of presentation, we will assume that $\sqrt{d}$ is an integer. $\rob$ is assumed to be advice-agnostic, i.e., the behavior of $\adv$ does not impact the decisions made by $\rob$ (nor does the behavior of $\alg$, since $\alg$ itself depends on both $\adv$ and $\rob$). We construct a lower bound in the spirit of the standard example of chasing faces of the hypercube. At a high level, the $\cbc$ instance we construct has two phases: the first is comprised of multiple subphases in which an affine subspace is chosen adversarially and is repeatedly served until $\rob$ has ``almost'' stopped moving. This phase lasts either until $3\sqrt{d}$ subphases have concluded, or until the first time that $\alg$ coincides with $\adv$ at the end of a subphase, whichever happens sooner. If the former holds, i.e., if $\alg$ ends each of the $3\sqrt{d}$ subphases at $\rob$, then the instance is done. Otherwise, the second phase begins: there are a few different cases, but generally, the same affine subspace is served repeatedly while $\adv$ slowly drifts away from $\rob$ until $\alg$ switches back to $\rob$. Then, the final body is simply the last advice decision as a singleton, forcing $\alg$ to move back to the advice, and the instance concludes. 

We now describe the lower bound in more specific detail. Since $\rob$ is advice-agnostic, we may begin by describing its behavior before specifying the behavior of $\adv$ and reasoning about the switching algorithm $\alg$. We denote by $\evec_j \in \R^d$, $j \in [d]$ the $j$\textsuperscript{th} standard unit basis vector, which is 1 in its $j$\textsuperscript{th} entry and 0 elsewhere. Choose any $\delta > 0$. The starting position is $\xvec_0 = \vec{0}$. 


\noindent\textbf{Phase one.} At time $t = 1$, the served body $K_1$ is the hyperplane forcing the first coordinate to be $z_1 \coloneqq 1$:
$$K_1 = \left\{\xvec : \xvec^\top \evec_1 = z_1\right\}.$$
This same hyperplane $K_1$ is then repeatedly served until the time $m_1$ at which $\rob$ is almost stationary. That is, the time $m_1 < \infty$ is chosen to satisfy the property that the cumulative cost incurred by $\rob$ after time $m_1$, if $K_1$ were repeated indefinitely thereafter, is bounded above by $\delta$. Such a time $m_1$ must exist, since $\rob$ is $b$-competitive, the offline optimal cost for the instance comprised of repeated $K_1$s is $1$, and the tail of a convergent series converges to zero. At time $m_1$, if $\alg$'s decision coincides with that of $\adv$, i.e. if $\alg_{m_1} = \adv_{m_1}$ (note we will define the behavior of $\alg$ and $\adv$ later on), then we say that phase one is complete and we move on to phase two below. Otherwise, we continue to the next subphase in phase one as follows.

Let $z_2 \coloneqq -\sgn(\rob_{m_1, 2})$ be the negative of the sign of $\rob$'s $2$\textsuperscript{nd} entry at time $m_1$ (defaulting to 1 if $\rob_{m_1, 2} = 0$). At time $t = m_1 + 1$, we serve a new affine subspace $K_2$ defined as
$$K_2 = \left\{\xvec : \xvec^\top \evec_i = z_i, i = 1, 2\right\}.$$
Note that this forces $\rob$ to incur cost at least 1 at time $m_1 + 1$. This same body is repeated until the time $m_2$ at which $\rob$ is almost stationary. That is, just as before, $m_2$ is defined as the time at which, if $K_2$ were repeated indefinitely from time $m_2+1$ onward, $\rob$ would incur total cost no more than $\delta$ after time $m_2$. For the same reason as before, $m_2 < \infty$ is certain to exist by $b$-competitiveness of $\rob$. If $\alg_{m_2} = \adv_{m_2}$, then we say that phase one is complete and move on to phase two below. Otherwise, we continue to the next subphase in phase one.

The remaining subphases in phase one are constructed similarly: for each $j = 3, \ldots, 3\sqrt{d}$, we define $z_j \coloneqq -\sgn(\rob_{m_{j-1}, j})$ to be the negative of the sign of $\rob$'s $j$\textsuperscript{th} entry at time $m_{j-1}$, and at time $t = m_{j-1}+1$, we serve a new affine subspace $K_j$ defined as
\begin{equation} \label{eq:K_j_definition}
    K_j = \left\{\xvec : \xvec^\top \evec_i = z_i, i = 1, \ldots, j\right\},
\end{equation}
which forces $\rob$ to incur cost at least 1. This body $K_j$ is then repeated until the time $m_j$ at which $\rob$ is almost stationary, i.e., after which it would incur cumulative cost no more than $\delta$, were $K_j$ to be repeated indefinitely. Then, if $\alg_{m_j} = \adv_{m_j}$, we say that phase one is complete and move on to phase two below. Otherwise, we remain in phase one and repeat this step with an incremented value of $j$. Once the subphase corresponding to $j = 3\sqrt{d}$ is completed, then the instance is concluded without moving on to phase two.

\noindent\textbf{Behavior of the advice.} We specify the behavior of $\adv$ based on the behavior of $\rob$ on the (possibly) counterfactual instance wherein phase one runs to termination without moving to phase two. That is, let $\rvec_1, \rvec_2, \ldots, \rvec_{m_{3\sqrt{d}}}$ be the decisions of $\rob$ on an auxiliary $\cbc$ instance where $K_1$ is served from time $1$ through $m_1$, $K_2$ is served from time $m_1 + 1$ through $m_2$, and so on, terminating with $K_{3\sqrt{d}}$ being served from time $m_{3\sqrt{d}-1}+1$ through $m_{3\sqrt{d}}$. Then define
\begin{equation} \label{eq:switch_lb_advice}
    \avec = \argmax_{\xvec \in \{\pm 1\}^{d-3\sqrt{d}}} \min_{j = 1, \ldots, 3\sqrt{d}} \|\xvec - \rvec_{m_j, 3\sqrt{d}+1:}\|_{\ell^2},
\end{equation}
where $\rvec_{m_j, 3\sqrt{d}+1:}$ is the vector obtained by dropping the first $3\sqrt{d}$ entries in $\rvec_{m_j}$. Thus, $\avec$ is the corner of the hypercube $\{\pm 1\}^{d-3\sqrt{d}}$ that is farthest (in $\ell^2$) from any of the subvectors comprised of the last $d-3\sqrt{d}$ entries of the decisions $\rvec_{m_1}, \ldots, \rvec_{m_{3\sqrt{d}}}$ made by $\rob$ at the conclusion of the phase one subphases. Then we define the advice's phase one behavior simply as follows: at time $1$, the advice immediately moves to the point
$$\hat{\avec} = (z_1, \ldots, z_{3\sqrt{d}}, a_1, \ldots, a_{d-3\sqrt{d}}),$$ 
and it remains there until phase one is completed.

\noindent\textbf{Phase two.} Fix $\epsilon > 0$. Suppose that phase one terminates at time $m_j$, where $j < 3\sqrt{d}$ (since if $j = 3\sqrt{d}$, then the instance ends without moving on to phase two). Thus it is the case that $\alg_{m_j} = \adv_{m_j} = \hat{\avec}$, and $\rob_{m_j} = \rvec_{m_j}$. Then the instance splits into two cases:
\begin{enumerate}[1.)]
    \item \label{switch_lb_case1} Suppose that $\|\avec - \rvec_{m_j, 3\sqrt{d}+1:}\|_{\ell^2} \geq \sqrt{d-3\sqrt{d}}$, and define $\vvec = \frac{\avec - \rvec_{m_j, 3\sqrt{d}+1:}}{\|\avec - \rvec_{m_j, 3\sqrt{d}+1:}\|_{\ell^2}}$. Then at each time $t = m_j + 1, \ldots, m_j + k$ (where $k$ will be defined later), we serve the body $K_j$ again. By our selection of $m_j$, the robust algorithm $\rob$ will remain $\delta$-close to its decision $\rob_{m_j}$, since we are simply continuing to serve the same body. However, at each of these times, we make the advice move to the point
    $$\adv_t = \left(z_1, \ldots, z_{3\sqrt{d}}, a_1 + (t-m_j)\epsilon v_1, \ldots, a_{d-3\sqrt{d}}+ (t-m_j)\epsilon v_{d-3\sqrt{d}}\right).$$
    That is, at each time $t = m_j + 1, \ldots, m_j+k$, the advice takes a step of length $\epsilon$ in the direction $\vvec$ in its last $d-3\sqrt{d}$ coordinates. Then $k$ is chosen such that $m_j + k$ is the first time after $m_j$ at which $\alg_{m_j + k} = \rob_{m_j + k}$, i.e., the first time at which the algorithm switches back to $\rob$ after following the advice. Note that $k < \infty$, by the assumption that $\alg$ has finite robustness. Then the final body is chosen as
    $$K_{\mathrm{fin}} = \left\{\adv_{m_j+k}\right\} = \left\{\left(z_1, \ldots, z_{3\sqrt{d}}, a_1 + k\epsilon v_1, \ldots, a_{d-3\sqrt{d}} + k\epsilon v_{d-3\sqrt{d}}\right)\right\},$$
    which allows the advice to stay put while $\rob$ and $\alg$ must move back to coincide with it.
    
    \item \label{switch_lb_case2} Suppose that $\|\avec - \rvec_{m_j, 3\sqrt{d}+1:}\|_{\ell^2} < \sqrt{d-3\sqrt{d}}$. Since $\avec$ maximizes the objective of \eqref{eq:switch_lb_advice}, then it must hold that
    \begin{equation} \label{min_dist}
        \min_{i=1, \ldots, 3\sqrt{d}}\|-\avec - \rvec_{m_i, 3\sqrt{d}+1:}\|_{\ell^2} < \sqrt{d-3\sqrt{d}}.
    \end{equation}
    Let $i^*$ be the minimizing index in \eqref{min_dist}; note that $i^* \neq j$, since otherwise,
    \begin{align*}
        2\sqrt{d-3\sqrt{d}} &= \|2\avec\| \\
        &\leq \|\avec - \rvec_{m_j, 3\sqrt{d}+1:}\|_{\ell^2} + \|\avec + \rvec_{m_j, 3\sqrt{d}+1:}\|_{\ell^2} & \text{by the triangle inequality}\\
        &< 2\sqrt{d-3\sqrt{d}}
    \end{align*}
    giving a contradiction. Then the instance splits into two further subcases:
    \begin{enumerate}[(a)]
        \item \label{switch_lb_casea} Suppose that $i^* < j$. Then, just as in case \ref{switch_lb_case1}, at each time $t = m_j + 1, \ldots, m_j + k$, we serve the body $K_j$ again. At each of these times, we make the advice move to the point 
        $$\adv_t = \left(z_1, \ldots, z_{3\sqrt{d}}, a_1 + (t-m_j)\epsilon v_1, \ldots, a_{d-3\sqrt{d}}+ (t-m_j)\epsilon v_{d-3\sqrt{d}}\right),$$
        where $\vvec = \frac{\avec - \rvec_{m_j, 3\sqrt{d}+1:}}{\|\avec - \rvec_{m_j, 3\sqrt{d}+1:}\|_{\ell^2}}$ just as in case \ref{switch_lb_case1}. Just as in case \ref{switch_lb_case1}, $k$ is chosen such that $m_j + k$ is the first time after $m_j$ at which $\alg_{m_j + k} = \rob_{m_j + k}$, i.e., the first time at which the algorithm switches back to $\rob$ after following the advice. Then the final body is chosen as
        $$K_{\mathrm{fin}} = \left\{\adv_{m_j+k}\right\} = \left\{\left(z_1, \ldots, z_{3\sqrt{d}}, a_1 + k\epsilon v_1, \ldots, a_{d-3\sqrt{d}} + k\epsilon v_{d-3\sqrt{d}}\right)\right\},$$
        which allows the advice to stay put while $\rob$ and $\alg$ must move back to coincide with it.

        \item \label{switch_lb_caseb} Suppose that $i^* > j$. Then for each $l = j+1, \ldots, i^*$, serve the body $K_l$ as defined in \eqref{eq:K_j_definition} from time $m_{l-1}+1$ through $m_l$, while keeping the advice at the same point $\hat{\avec}$. Since $\rob$ is advice agnostic and this sequence of bodies coincides with the remainder of the phase one sequence of bodies, it will be the case that $\rob_{m_{i^*}} = \rvec_{m_{i^*}}$. Then, finally, we split into two further subcases.
        \begin{enumerate}[(i), wide, labelwidth=!, labelindent=0pt]
            \item \label{switch_lb_casei} If $\alg_{m_{i^*}} = \rob_{m_{i^*}} = \rvec_{m_{i^*}}$, then simply choose the final body as 
            $$K_{\mathrm{fin}} = \left\{\hat{\avec}\right\},$$
            which allows the advice to stay put while $\rob$ and $\alg$ must move back to coincide with it.
            
            \item \label{switch_lb_caseii} If $\alg_{m_{i^*}} = \adv_{m_{i^*}} = \hat{\avec}$, then proceed similarly to subcase \ref{switch_lb_casea}: for each time $t = m_{i^*} + 1, \ldots, m_{i^*} + k$, we serve the body $K_{i^*}$ again and make the advice move to the point
            $$\adv_t = \left(z_1, \ldots, z_{3\sqrt{d}}, a_1 + (t-m_j)\epsilon v_1, \ldots, a_{d-3\sqrt{d}}+ (t-m_j)\epsilon v_{d-3\sqrt{d}}\right),$$
            where this time $\vvec = \frac{\avec - \rvec_{m_{i^*}, 3\sqrt{d}+1:}}{\|\avec - \rvec_{m_{i^*}, 3\sqrt{d}+1:}\|_{\ell^2}}$.
            Just as in subcase \ref{switch_lb_casea}, $k$ is chosen such that $m_{i^*} + k$ is the first time after $m_{i^*}$ at which $\alg_{m_{i^*}+k} = \rob_{m_{i^*}+k}$, i.e., the first time at which the algorithm switches back to $\rob$ after following the advice starting from time $m_{i^*}$. Then the final body is simply chosen as
            $$K_{\mathrm{fin}} = \left\{\adv_{m_j+k}\right\} = \left\{\left(z_1, \ldots, z_{3\sqrt{d}}, a_1 + k\epsilon v_1, \ldots, a_{d-3\sqrt{d}} + k\epsilon v_{d-3\sqrt{d}}\right)\right\},$$
            which allows the advice to stay put while forcing $\rob$ and $\alg$ to move to coincide with it.
        \end{enumerate}
    \end{enumerate}

\end{enumerate}

\noindent\textbf{Cost analysis.} Let us now tally costs for each of the cases of the instance to prove the result.

Consider the initial case where the instance never makes it out of phase one; this means that $3\sqrt{d}$ subphases occur in phase one, and $\alg$ finishes each subphase at the $\rob$ decision. Since $\|\rvec_{m_j} - \rvec_{m_{j-1}}\|_{\ell^2} \geq 1$ for each $j = 1, \ldots, 3\sqrt{d}$ (where $\rvec_{m_0} \coloneqq \rvec_0 = \xvec_0$), this means that $\alg$ incurs cost at least $3\sqrt{d}$, whereas the advice, which moves immediately to $\hat{\avec} \in \{\pm 1\}^d$ and stays there throughout the entire instance, incurs cost $\sqrt{d}$. Thus $\alg$ is at least 3-consistent, and we are done.

Now, we turn to each of the cases within which the instance makes it to phase two. First, consider case \ref{switch_lb_case1}. Since $\alg_{m_j} = \adv_{m_j} = \hat{\avec}$, the cost incurred by $\alg$ through time $m_j$ is at least $\sqrt{d}$. Then from time $m_j$ to $m_j+k-1$ while $\alg$ is following the advice, $\alg$ incurs cost $\|(k-1)\epsilon \vvec\|_{\ell^2} = (k-1)\epsilon$. At time $m_j + k$, $\alg$ switches back to $\rob$, incurring cost at least
\begin{align*}
    \|\adv_{m_j+k-1} - \rob_{m_j+k}\|_{\ell^2} &\geq \|\adv_{m_j+k-1} - \rob_{m_j}\|_{\ell^2} - \|\rob_{m_j} - \rob_{m_j + k}\|_{\ell^2} \\
    &\geq \|\adv_{m_j+k-1} - \rob_{m_j}\|_{\ell^2} - \delta \\
    &\geq \|\adv_{m_j + k - 1, 3\sqrt{d}+1:} - \rob_{m_j, 3\sqrt{d}+1:}\|_{\ell^2} - \delta\\
    &= \|(\avec + (k-1)\epsilon\vvec) - \rvec_{m_j, 3\sqrt{d}+1:}\|_{\ell^2} - \delta\\
    &= \|\avec - \rvec_{m_j, 3\sqrt{d}+1:}\|_{\ell^2} + (k-1)\epsilon - \delta\\
    &\geq \sqrt{d-3\sqrt{d}} + (k-1)\epsilon - \delta. \tageq\label{ineq:adv_to_rob_switch}
\end{align*}
Finally, by an analogous argument to \eqref{ineq:adv_to_rob_switch}, to switch back to $\adv$, $\alg$ incurs a cost of at least $\sqrt{d-3\sqrt{d}} + k\epsilon - \delta$. In sum, $\alg$ incurs a total cost of $\sqrt{d} + 2\sqrt{d - 3\sqrt{d}} + (3k - 2)\epsilon - 2\delta$. On the other hand, $\adv$ incurs a total cost of $\sqrt{d} + k\epsilon$. Then the consistency of $\alg$ is
$$\frac{\sqrt{d} + 2\sqrt{d - 3\sqrt{d}} + (3k - 2)\epsilon - 2\delta}{\sqrt{d} + k\epsilon}$$
which can be made arbitrarily close to 3 by choosing $\epsilon$ and $\delta$ small and taking $d$ arbitrarily large. 

Next, let's move to case \ref{switch_lb_case2}. First, we set up some preliminaries. Let's call $\rho = \sqrt{d - 3\sqrt{d}} - \|\avec - \rvec_{m_j, 3\sqrt{d}+1:}\|_{\ell^2}$, and note that $\rho > 0$. Since $\|\avec - \rvec_{m_j, 3\sqrt{d}+1:}\|_{\ell^2} \geq \|-\avec - \rvec_{m_{i^*}, 3\sqrt{d}+1:}\|_{\ell^2}$, we have that 
\begin{equation} \label{ineq:rvec_in_ball}
    \|-\avec - \rvec_{m_{i^*}, 3\sqrt{d}+1:}\|_{\ell^2} \leq \sqrt{d-3\sqrt{d}}-\rho
\end{equation}
and hence
\begin{align*}
    \|\rvec_{m_{i^*}, 3\sqrt{d}+1:}\|_{\ell^2} &\geq \|-\avec\| - \|-\avec - \rvec_{m_{i^*}, 3\sqrt{d}+1:}\|_{\ell^2} \\
    &\geq \rho \tageq\label{ineq:rvec_magnitude}
\end{align*}
Then
\begin{align*}
    \|\hat{\avec} - \rvec_{m_{i^*}}\|_{\ell^2} &\geq \|\avec - \rvec_{m_{i^*}, 3\sqrt{d}+1:}\|_{\ell^2} \tageq\label{ineq:subvector}\\
    &\geq \left\|\avec - \Pi_{B(-\avec, \sqrt{d-3\sqrt{d}}-\rho)}\avec\right\|_{\ell^2} \tageq\label{ineq_proj_nonexpansivity}\\
    &= \left\|2\avec - \Pi_{B(\vec{0}, \sqrt{d-3\sqrt{d}}-\rho)}2\avec\right\|_{\ell^2} \tageq\label{eq:translation_projection}\\
    &= \left\|2\avec - (\sqrt{d-3\sqrt{d}}-\rho)\frac{\avec}{\|\avec\|_{\ell^2}}\right\|_{\ell^2} \tageq\label{eq:proj_radial_proj}\\
    &= \left\|\avec + \rho\frac{\avec}{\|\avec\|_{\ell^2}}\right\|_{\ell^2} \\
    &= \sqrt{d-3\sqrt{d}}+\rho \tageq\label{ineq:avec_rvec_dist}
\end{align*}
where $\Pi_K\xvec$ denotes the projection of the point $\xvec$ onto the convex body $K$,
\eqref{ineq:subvector} follows from $\avec$ and $\rvec_{m_{i^*}, 3\sqrt{d}+1:}$ being subvectors of $\hat{\avec}$ and $\rvec_{m_{i^*}}$, respectively, \eqref{ineq_proj_nonexpansivity} follows from \eqref{ineq:rvec_in_ball} and non-expansivity of the projection, \eqref{eq:translation_projection} follows from translation, \eqref{eq:proj_radial_proj} applies the fact that the projection onto an origin-centered ball is just a radial projection, and \eqref{ineq:avec_rvec_dist} follows from $\|\avec\|_{\ell^2} = \sqrt{d-3\sqrt{d}}$.

Now, let's consider the subcases, starting with subcase \ref{switch_lb_casea}. Since $i^* < j$, we know that $\alg_{m_{i^*}} = \rob_{m_{i^*}} = \rvec_{m_{i^*}}$. Then by \eqref{ineq:rvec_magnitude}, $\alg$ incurs cost at least $\rho$ to get to $\rvec_{m_{i^*}}$, and by \eqref{ineq:avec_rvec_dist} it incurs another cost of at least $\sqrt{d-3\sqrt{d}}+\rho$ to get to $\adv_{m_j} = \hat{\avec}$. From time $m_j$ to $m_j + k - 1$ while $\alg$ is following the advice, $\alg$ incurs cost $(k-1)\epsilon$. Then at time $m_j + k$, $\alg$ switches back to $\rob$, and by a similar analysis to that in \eqref{ineq:adv_to_rob_switch} done for case \ref{switch_lb_case1}, it incurs cost at least $\sqrt{d - 3\sqrt{d}} - \rho + (k-1)\epsilon - \delta$ to do so. Finally, to switch back to $\adv$, $\alg$ incurs a cost of at least $\sqrt{d - 3\sqrt{d}} - \rho + k\epsilon - \delta$. Then in sum, $\alg$ has incurred a total cost of $3\sqrt{d-3\sqrt{d}} + (3k-2)\epsilon - 2\delta$ in this instance case. On the other hand, $\adv$ incurs a total cost of $\sqrt{d}+k\epsilon$, so $\alg$ has consistency
$$\frac{3\sqrt{d-3\sqrt{d}} + (3k-2)\epsilon - 2\delta}{\sqrt{d}+k\epsilon},$$
which can be made arbitrarily close to 3 by choosing $\epsilon$ and $\delta$ small and taking $d$ arbitrarily large.

Now, we move to subcase \ref{switch_lb_caseb}, beginning first with \ref{switch_lb_casei}. $\alg$ spends $\sqrt{d}$ to get to $\adv_{m_j} = \hat{\avec}$ in the first place, and then by \eqref{ineq:avec_rvec_dist} it spends cost at least $\sqrt{d - 3\sqrt{d}} + \rho$ to get to $\rob_{m_{i^*}}$ at time $m_{i^*}$. Finally, it spends at least another $\sqrt{d - 3\sqrt{d}} + \rho$ to get back to $\hat{\avec}$ for the final timestep. Thus in sum, $\alg$ incurs cost $\sqrt{d} + 2(\sqrt{d - 3\sqrt{d}} + \rho)$, whereas $\adv$ incurs cost $\sqrt{d}$, giving a consistency of
$$\frac{\sqrt{d} + 2(\sqrt{d - 3\sqrt{d}} + \rho)}{\sqrt{d}},$$
which even for arbitrarily small $\rho > 0$ can be made arbitrarily close to 3 by choosing $d$ sufficiently large.

Finally, we consider scenario \ref{switch_lb_caseii} in subcase \ref{switch_lb_caseb}. $\alg$ first spends $\sqrt{d}$ to get to $\adv_{m_j} = \hat{\avec}$, and then from time $m_{i^*}+1$ through $m_{i^*}+k-1$ it incurs cost $(k-1)\epsilon$ to follow the advice. Using \eqref{ineq:avec_rvec_dist} and reasoning analogous to that in \eqref{ineq:adv_to_rob_switch}, $\alg$ incurs cost $\sqrt{d - 3\sqrt{d}}+\rho+(k-1)\epsilon-\delta$ to switch back to $\rob$ at time $m_{i^*}+k$, and finally, it incurs cost $\sqrt{d - 3\sqrt{d}}+\rho+k\epsilon-\delta$ to switch back to the advice in the final timestep. Thus in sum, $\alg$ incurs cost at least $\sqrt{d} + 2(\sqrt{d - 3\sqrt{d}}+\rho) + (3k-2)\epsilon - 2\delta$, while $\adv$ incurs cost $\sqrt{d}+k\epsilon$. Thus $\alg$ has consistency
$$\frac{\sqrt{d} + 2(\sqrt{d - 3\sqrt{d}}+\rho) + (3k-2)\epsilon - 2\delta}{\sqrt{d}+k\epsilon},$$
which, even for very small $\rho > 0$, can be made arbitrarily close to 3 by choosing $\epsilon, \delta$ small and taking $d$ sufficiently large. 

\jmlrQED

\subsection{Proof of Proposition \ref{proposition:switching_steiner} \label{appendix:switching_steiner_proof}}

Let us first recall Theorem 2.1 of \cite{bubeck_chasing_2019}, which characterizes the cost incurred by moving to the Steiner point of each nested body. 

\begin{theorem}[{\cite[Theorem 2.1]{bubeck_chasing_2019}}] \label{thm:steiner_performance}
    Let $\xvec_0 = \vec{0}$ and $K_1 \subseteq B(\vec{0}, r)$ for some $r > 0$. Then following the Steiner point of each nested body $K_t$ incurs total movement cost no more than $rd$. 
\end{theorem}

We now prove Proposition \ref{proposition:switching_steiner}. 
For clarity, we abbreviate $\nestedswitch$ in this proof as $\ns$.

If $\ns$ only ever follows $\adv$, then $\epsilon \cdot \cost_{\adv} < r(d+2)$ and $\cost_{\ns} = \cost_{\adv}$, so $\cost_\ns \leq \frac{r(d+2)}{\epsilon}$. Thus $\ns$ is 1-competitive with respect to $\adv$ and $\frac{r(d+2)}{\epsilon}$-robust, since $\cost_\opt \geq 1$.

On the other hand, if $\ns$ only ever follows $\rob$, then $\cost_\ns = \cost_\rob$ and $\epsilon \cdot \cost_\adv \geq r(d+2)$. Since $\rob$ just follows the Steiner point of each nested body, we have $\cost_\rob \leq r + rd$, where the $rd$ comes from Theorem \ref{thm:steiner_performance} and the extra factor of $r$ arises from the triangle inequality applied to the $t = 1$ movement:
$$\|\svec_1 - \xvec_0\|_{\ell^2} \leq \|\svec_1 - \yvec\|_{\ell^2} + \|\yvec - \xvec_0\|_{\ell^2} \leq \|\svec_1 - \yvec\|_{\ell^2} + r.$$
Thus $\cost_\ns = \cost_\rob \leq r(d+1) \leq (1+\epsilon)\cost_\adv$, and the desired robustness also holds. 

Finally, suppose $\ns$ switches to $\rob$ at time $t \in [T]$; i.e., $\ns_1 = \adv_1, \ldots, \ns_{t-1} = \adv_{t-1}, \ns_t = \rob_t, \ldots, \ns_T = \rob_T$. We know that
$$\cost_\ns(1, t-1) = \cost_{\adv}(1, t-1) < \frac{r(d+2)}{\epsilon}$$
and since $K_t \subseteq B(\yvec, r)$,
$$\cost_\ns(t, t) = \|\rob_t - \adv_{t-1}\|_{\ell^2} \leq 2r$$
and finally
$$\cost_\ns(t+1, T) = \cost_{\rob}(t+1, T) \leq rd.$$
Thus in sum, 
$$\cost_\ns \leq \frac{r(d+2)}{\epsilon} + rd + 2r = \left(1 + \frac{1}{\epsilon}\right)r(d+2).$$
This gives both the robustness and consistency bounds, since $\epsilon \cdot \cost_\adv \geq r(d + 2)$ and $\cost_\opt \geq 1$. \jmlrQED

\section{Background from the geometry of normed vector spaces \label{appendix:nvs_geometry}}

In this appendix, we introduce some notions and results from the literature on the geometry of normed vector spaces, expanding on the brief definitions of the rectangular constant and the radial retraction given in the main text in Section \ref{section:main_bicompetitive_algorithm}. In the following definitions and results, $\calX = (X, \|\cdot\|)$ is an arbitrary real normed vector space. 

We begin by defining Birkhoff-James orthogonality, which generalizes the usual Hilbert space orthogonality.

\begin{definition}[{\cite[p.169]{birkhoff_orthogonality_1935}}; {\cite[p.265]{james_orthogonality_1947}}]
    $\xvec \in X$ is \textbf{Birkhoff-James orthogonal} to $\yvec \in X$, denoted $\xvec \perp \yvec$, if $\|\xvec\| \leq \|\xvec + \lambda \yvec\|$ for all $\lambda \in \R$.
\end{definition}

Note that, unlike orthogonality in Hilbert spaces, Birkhoff-James orthogonality is not generally symmetric. However, it is homogeneous.

\begin{lemma}[{\cite[p.265]{james_orthogonality_1947}}; {\cite[Remark 1]{joly_caracterisations_1969}}] \label{lemma:bj_orthog}
    If $\xvec \perp \yvec$, then $a \xvec \perp b \yvec$ for all $a, b \in \R$.
\end{lemma}

Using Birkhoff-James orthogonality, we can formally define define the first constant we introduced in Section \ref{section:main_bicompetitive_algorithm}: the rectangular constant. It is motivated by the following observation: in a finite-dimensional inner product space, orthogonality of $\xvec$ and $\yvec$ implies that $\frac{\|\xvec\| + \|\yvec\|}{\|\xvec + \yvec\|} \leq \sqrt{2}$. In an arbitrary normed vector space, the upper bound $\sqrt{2}$ is replaced with the \emph{rectangular constant}, defined as follows using Birkhoff-James orthogonality.

\begin{definition}[{\cite[Definition 2]{desbiens_constante_1990}}; original from {\cite[Definition 2]{joly_caracterisations_1969}}]
    The \textbf{rectangular constant} $\mu(\calX)$ of a real normed vector space $\calX$ is defined as
    $$\mu(\calX) = \sup_{\xvec \perp \yvec} \frac{\|\xvec\| + \|\yvec\|}{\|\xvec + \yvec\|}.$$
\end{definition}

It is known that $\sqrt{2} \leq \mu(\calX) \leq 3$ \cite[Section II]{joly_caracterisations_1969}, and these bounds are tight: $\mu(\calX) = \sqrt{2}$ for any Hilbert space \cite[Example 1; Section III] {joly_caracterisations_1969}, and $\mu(\calX) = 3$ for ``nonuniformly nonsquare'' spaces such as $\ell^1$ and $\ell^\infty$ (\cite{baronti_revisiting_2021}). Moreover, $\mu(\ell^p) < 3$ for all $p \in (1, \infty)$. In fact, tighter bounds are known for the $\ell^p$ spaces: we review these in the following theorem.

\begin{theorem}[{\cite[Theorems 5.2, 5.4, 5.5]{baronti_revisiting_2021}}] \label{thm:mu_lp}
    For $1 < p \leq 2$,
    $$\mu(\ell^p) \leq \min\left\{\left(1 + \left(2^{1/(p-1)}-1\right)^{p-1}\right)^{1/p}, \sqrt{\frac{p}{p-1}}\right\}.$$
    For $p \geq 2$,
    $$\mu(\ell^p) \leq \left(1+\left(2^{p-1}-1\right)^{1/(p-1)}\right)^{(p-1)/p}.$$
\end{theorem}

Together, these constitute an upper bound on $\mu(\ell^p)$ that attains a (tight) minimum of $\sqrt{2}$ at $p = 2$, and that continuously increases toward 3 as $p \to \infty$ and $p \to 1$. 

We now reiterate the definition of the radial retraction and its Lipschitz constant given in Section \ref{section:main_bicompetitive_algorithm}. 

\begin{definition}[\cite{rieffel_lipschitz_2006}]
    On a normed vector space $\calX = (X, \|\cdot\|)$, the radial retraction  $\rho(\cdot\,; r) : X \to B(\bv{0}, r)$ is the metric projection onto the closed ball of radius $r \geq 0$:
    $$\rho(\xvec; r) = \begin{cases} \xvec & \text{if $\|\xvec\| \leq r$} \\ r\frac{\xvec}{\|\xvec\|} & \text{if $\|\xvec\| > r$.} \end{cases}$$
    
    We define $k(\calX)$ to be the Lipschitz constant of $\rho(\cdot\,; 1)$, i.e., the smallest real number satisfying
    $$\|\rho(\xvec; 1) - \rho(\yvec; 1)\| \leq k(\calX)\|\xvec - \yvec\|$$
    for all $\xvec, \yvec \in X$.
    
\end{definition}

It holds that $k(\calX)$ is bounded between 1 and 2 in any normed vector space $\calX$ (\cite{thele_results_1974}). Moreover, $k(\calX)$ is identically the Lipschitz constant of $\rho(\cdot\,; r)$ for any $r > 0$ (\cite{rieffel_lipschitz_2006}). To see that this is the case, observe that $\rho(\xvec; r) = r\cdot\rho(\frac{\xvec}{r}; 1)$; it then follows that
$$\|\rho(\xvec; r) - \rho(\yvec; r)\| = r\left\|\rho\left(\frac{\xvec}{r}; 1\right) - \rho\left(\frac{\yvec}{r}; 1\right)\right\| \leq k(\calX)\|\xvec - \yvec\|.$$
Thus $k(\calX)$ is an upper bound on the Lipschitz constant of $\rho(\xvec; r)$ for general $r > 0$. Similar reasoning shows that $\rho(\xvec; r)$ can have no smaller Lipschitz constant than $k(\calX)$; so $k(\calX)$ is the Lipschitz constant for all $\rho(\xvec; r)$, $r > 0$. 

We conclude this section with a result relating $k(\calX)$ with $\mu(\calX)$. 

\begin{proposition} \label{prop:lipschitz_rectangular_relation}
    On a real normed vector space $\calX = (X, \|\cdot\|)$, it holds that $k(\calX) \leq \mu(\calX)$.
\end{proposition}
\begin{proof}
    \cite[Theorem 1]{thele_results_1974} characterizes $k(\calX)$ as follows:
    \begin{equation} \label{eq:thele_lipschitz}
        k(\calX) = \sup_{\xvec \perp \yvec, \yvec \neq \vec{0}, \lambda \in \R} \frac{\|\yvec\|}{\|\yvec - \lambda \xvec\|}.
    \end{equation}
    Since, by Lemma \ref{lemma:bj_orthog}, Birkhoff-James orthogonality is homogeneous, it is straightforward to see that (\ref{eq:thele_lipschitz}) can be equivalently expressed as
    $$k(\calX) = \sup_{\xvec \perp \yvec, \yvec \neq \vec{0}} \frac{\|\yvec\|}{\|\xvec+\yvec\|}.$$
    Then it is clear that
    $$k(\calX) = \sup_{\xvec \perp \yvec, \yvec \neq \vec{0}} \frac{\|\yvec\|}{\|\xvec+\yvec\|} \leq \sup_{\xvec \perp \yvec} \frac{\|\xvec\| + \|\yvec\|}{\|\xvec + \yvec\|} = \mu(\calX).$$
\end{proof}

\section{Proof of Theorem \ref{theorem:consistent_robust_general} \label{appendix:main_general_theorem_proof}}

\subsection{Geometric lemmas \label{appendix:geometric_lemmas}}
Before presenting the analysis of Algorithm \ref{alg:interp}, we take a brief foray into the geometric theory of normed vector spaces, presenting and proving some lemmas that will be helpful in proving the bicompetitive bound given in Theorem \ref{theorem:consistent_robust_general}. The results in this section depend heavily on the definitions and results introduced in Appendix \ref{appendix:nvs_geometry}.

The first lemma characterizes (a modified form of) the radial retraction as a metric projection onto the boundary of a closed ball.

\begin{lemma} \label{lemma:metric_proj}
    Let $(X, \|\cdot\|)$ be a normed vector space, and consider arbitrary $r \geq 0$, $\xvec \in X$, and $\yvec \in X \setminus \{\xvec\}$. Define $\hat{\yvec} = \xvec+r\frac{\yvec-\xvec}{\|\yvec-\xvec\|}$. Then $\|\yvec - \hat{\yvec}\| \leq \|\yvec - \wvec\|$ for all $\wvec \in \partial B(\xvec, r)$.
\end{lemma}
\begin{proof}
    It suffices to consider the case when $\xvec = \bv{0}$. If $r = 0$, $\partial B(\vec{0}, r) = \{\vec{0}\}$, so the result is clear. Otherwise, fix arbitrary $\wvec \in \partial B(\bv{0}, r)$ and observe
    \begin{align*}
        \|\yvec - \wvec\| &\geq |\|\yvec\| - \|\wvec\|| & \text{by the triangle inequality}\\
        &= |\|\yvec\| - r| \\
        &= \left|\|\yvec\| - \|\hat{\yvec}\|\right| \\
        &= \|\yvec - \hat{\yvec}\|
    \end{align*}
    where the last step follows from collinearity of $\yvec, \hat{\yvec}$, and $\bv{0}$.
\end{proof}

The second lemma generalizes the following geometric fact in the Euclidean plane to an arbitrary normed vector spaces: given a triangle $\avec\bvec\cvec$ in $(\R^2, \|\cdot\|_{\ell^2})$, and points $\xvec \in [\avec, \bvec], \yvec \in [\avec, \cvec]$ with $\|\xvec - \bvec\|_{\ell^2} = \|\yvec - \cvec\|_{\ell^2}$, it holds that $\|\xvec - \yvec\|_{\ell^2} \leq \|\bvec - \cvec\|_{\ell^2}$. In the general setting, this becomes a statement about the distance between the radial retractions of a single point onto two balls of the same radius with different centers.

\begin{lemma} \label{lemma:triangle_end_balls}
    Let $(X, \|\cdot\|)$ be a normed vector space, and fix arbitrary $\avec, \bvec, \cvec \in X$ and $r \geq 0$. Define $\xvec = \bvec + \rho(\avec - \bvec; r)$ and $\yvec = \cvec + \rho(\avec - \cvec; r)$. Then $\|\xvec - \yvec\| \leq \|\bvec - \cvec\|$. 
\end{lemma}
\begin{proof}
    We may assume without loss of generality that $\|\avec - \bvec\| \geq \|\avec - \cvec\|$. If $\bvec = \cvec$, $\|\xvec-\yvec\| = 0 = \|\bvec-\cvec\|$. Thus we restrict to the case where $\bvec \neq \cvec$ and distinguish cases based on the value of $r$. We may further restrict to the case where $\|\avec - \cvec\| > 0$, as the case $\avec = \cvec$ is trivial.

    If $r = 0$, then $\xvec = \bvec$ and $\yvec = \cvec$. Thus $\|\xvec - \yvec\| = \|\bvec - \cvec\|$. On the other hand, if $r \geq \|\avec - \bvec\|$, then $r \geq \|\avec - \cvec\|$ as well, so $\xvec = \yvec = \avec$, and certainly $\|\xvec - \yvec\| = 0 \leq \|\bvec - \cvec\|$. 
    
    Next, suppose $\|\avec - \cvec\| \leq r < \|\avec - \bvec\|$. Then $\yvec = \avec$, so $\|\xvec - \yvec\| = \|\xvec - \avec\|$. Moreover, $\|\xvec - \bvec\| = r \geq \|\avec - \cvec\|$. Then by the triangle inequality,
    \begin{align*}
    \|\bvec - \cvec\| &\geq |\|\avec - \bvec\| - \|\avec - \cvec\|| \\
    &= |\|\avec - \xvec\| + \|\xvec - \bvec\| - \|\avec - \cvec\|| \\
    &\geq \|\avec - \xvec\| \\
    &= \|\xvec - \yvec\|.
    \end{align*}

    Finally, suppose $0 < r < \|\avec - \cvec\|$, and define $\lambda = 1 - \frac{r}{\|\avec - \cvec\|}$. Since $\yvec + \bvec - \cvec = \bvec + r\frac{\avec-\cvec}{\|\avec-\cvec\|}$, we know that $\yvec + \bvec - \cvec \in \partial B(\bvec, r)$. Observe moreover that 
    $$\zvec \coloneqq \yvec + \lambda(\bvec-\cvec) = \bvec + r\frac{\avec - \bvec}{\|\avec - \cvec\|} \in [\avec, \bvec]$$
    and $\zvec \neq \bvec$ by assumption that $r > 0$. Thus $\frac{\zvec-\bvec}{\|\zvec-\bvec\|} = \frac{\avec - \bvec}{\|\avec - \bvec\|}$, so $\xvec = \bvec + r\frac{\zvec-\bvec}{\|\zvec-\bvec\|}$. Thus:
    \begin{align*}
        \|\bvec - \cvec\| &= \|(\yvec + \bvec - \cvec) - \yvec\| \\
        &= \|(\yvec+\bvec-\cvec) - \zvec\| + \|\zvec - \yvec\| &\text{by collinearity of $\yvec, \zvec, \yvec+\bvec-\cvec$} \\
        &\geq \|\xvec - \zvec\| + \|\zvec - \yvec\| &\text{applying Lemma \ref{lemma:metric_proj}} \tageq\label{eq:end_ball_lemma_step} \\
        &\geq \|\xvec - \yvec\| &\text{by triangle inequality.}
    \end{align*}
    where, in (\ref{eq:end_ball_lemma_step}), $\xvec, \yvec, \hat{\yvec}, \wvec,$ and $r$ in Lemma \ref{lemma:metric_proj} are instantiated during its invocation with this proof's $\bvec, \zvec, \xvec, (\yvec + \bvec - \cvec),$ and $r$, respectively.
\end{proof}

The next geometric lemma provides a bound on the total distance traveled first between two points on a sphere, and then from the second point to a scaled version thereof, in terms of the rectangular constant and the distance between the initial and final points.


\begin{lemma} \label{lemma:rect_const_ball}
    Let $(X, \|\cdot\|)$ be a normed vector space, and let $t > 1$, $r > 0$, and $\xvec, \yvec \in \partial B(\bv{0}, r)$. Then
    $$\|\yvec - \xvec\| + (t-1)\|\yvec\| \leq \mu(\calX)\|t\yvec - \xvec\|.$$
\end{lemma}
\begin{proof}
    By a corollary of the Hahn-Banach theorem \cite[Chapter 3, Corollary 7]{bollobas_linear_1999}, there exists a \emph{support functional} $f \in X^*$ at $\yvec$, i.e., some bounded linear functional $f : X \to \R$, with $\|f\|_{X^*} = 1$, $f(\yvec) = \|\yvec\| = r$, and the property that the hyperplane $H(r) \coloneqq \{\zvec \in X : f(\zvec) = r\}$ contains no points in $\text{int}(B(\bv{0}, r))$. Note that we can equivalently write $H(r)$ in affine subspace form $H(r) = \yvec + \ker(f) = \{\zvec \in X : \zvec = \yvec + \hvec, \hvec \in \ker(f)\}$, by linearity of $f$. The fact that $H(r)$ contains no points in the interior of $B(\bv{0}, r)$ means that $\yvec \perp \hvec$ for all $\hvec \in \ker(f)$.
    
    Define $s = f(\xvec)$, and note that since $\|\xvec\| = r$ and $\|f\|_{X^*} = 1$, we must have $s \leq r$. Then define $\zvec = \frac{s}{r}\yvec$, and observe $H(s) = \xvec + \ker(f) = \zvec + \ker(f)$. Thus $\xvec = \zvec + \hvec$ for some specific $\hvec \in \ker(f)$. By homogeneity of Birkhoff-James orthogonality (Lemma \ref{lemma:bj_orthog}), it follows that $(t-\frac{s}{r})\yvec \perp -\hvec$. As such,
    \begin{align*}
        \|\yvec - \xvec\| + (t-1)\|\yvec\| &\leq \|\yvec - \zvec\| + \|\xvec - \zvec\| + (t-1)\|\yvec\| \\
        &= \left(1 - \frac{s}{r}\right)\|\yvec\| + \|\hvec\| + (t-1)\|\yvec\| \\
        &= \|-\hvec\| + \left(t-\frac{s}{r}\right)\|\yvec\| \\
        &\leq \mu(\calX)\left\|-\hvec + \left(t-\frac{s}{r}\right)\yvec\right\| \\
        &= \mu(\calX)\|t\yvec - \xvec\|.
    \end{align*}
\end{proof}

Finally, we present a lemma building upon Lemma \ref{lemma:rect_const_ball} that will be indispensable for the consistency analysis of Algorithm \ref{alg:interp}.

\begin{lemma} \label{lemma:ball_projection_lemma}
    Let $(X, \|\cdot\|)$ be a normed vector space, and fix arbitrary $r \geq 0$, $\wvec, \yvec \in X$, and $\xvec \in X\setminus \mathrm{int}(B(\wvec, r))$. Define $\hat{\xvec} = \wvec + \rho(\xvec - \wvec; r)$ and $\hat{\yvec} = \wvec + \rho(\yvec - \wvec; r)$. Then
    $$\|\hat{\yvec} - \hat{\xvec}\| + \|\yvec - \hat{\yvec}\| \leq \mu(\calX)\|\yvec - \xvec\| + \|\xvec - \hat{\xvec}\|.$$
\end{lemma}
\begin{proof}
    If $r = 0$, then $B(\wvec, r) = \{\wvec\}$, so $\hat{\xvec} = \hat{\yvec} = \wvec$, and the result follows from the triangle inequality, as $\mu(\calX) \geq \sqrt{2}$. Thus we restrict to the case that $r > 0$.

    It suffices to consider the case where $\wvec = \bv{0}$. Then $\hat{\xvec} = \rho(\xvec; r)$ and $\hat{\yvec} = \rho(\yvec; r)$. We distinguish two cases.

    First, suppose $\yvec \in B(\bv{0}, r)$. Then $\hat{\yvec} = \yvec$, and by the triangle inequality,
    $$\|\hat{\yvec} - \hat{\xvec}\| + \|\yvec - \hat{\yvec}\| = \|\yvec - \hat{\xvec}\| \leq \|\yvec - \xvec\| + \|\xvec - \hat{\xvec}\| \leq \mu(\calX)\|\yvec - \xvec\| + \|\xvec - \hat{\xvec}\|.$$

    Second, suppose $\yvec \in X\setminus B(\bv{0}, r)$. Then $\hat{\yvec} = r\frac{\yvec}{\|\yvec\|}$. We distinguish two further subcases.

    \begin{enumerate}[(i)]
        \item Suppose that $\|\yvec - \hat{\yvec}\| \leq \|\xvec - \hat{\xvec}\|$. Then
        \begin{align*}
            \|\hat{\yvec} - \hat{\xvec}\| + \|\yvec - \hat{\yvec}\| &\leq \|\hat{\yvec} - \hat{\xvec}\| + \|\xvec - \hat{\xvec}\| \\
            &\leq k(\calX)\|\yvec - \xvec\| + \|\xvec - \hat{\xvec}\| \\
            &\leq \mu(\calX)\|\yvec - \xvec\| + \|\xvec - \hat{\xvec}\| &\text{by Proposition \ref{prop:lipschitz_rectangular_relation}.}
        \end{align*}

        \item On the other hand, suppose that $\|\yvec - \hat{\yvec}\| > \|\xvec - \hat{\xvec}\|$, or equivalently, $\|\yvec\| \geq \|\xvec\|$, and define $\zvec = \rho(\yvec; \|\xvec\|) = \|\xvec\|\frac{\yvec}{\|\yvec\|}$. By collinearity, $\|\zvec - \hat{\yvec}\| = \|\xvec\| - r = \|\xvec - \hat{\xvec}\|$. Furthermore, we have that 
        $$\|\hat{\yvec} - \hat{\xvec}\| \leq \frac{\|\xvec\|}{r}\|\hat{\yvec} - \hat{\xvec}\| = \|\zvec - \xvec\|.$$ 
        It then follows that
        \begin{align*}
            \|\hat{\yvec} - \hat{\xvec}\| + \|\yvec - \hat{\yvec}\| &\leq \|\zvec - \xvec\| + \|\yvec - \zvec\| + \|\zvec - \hat{\yvec}\| \\
            &= \|\zvec - \xvec\| + \left(\frac{\|\yvec\|}{\|\xvec\|} - 1\right)\|\zvec\| + \|\xvec - \hat{\xvec}\| \\
            &\leq \mu(\calX)\|\yvec - \xvec\| + \|\xvec - \hat{\xvec}\|
        \end{align*}
        where the final inequality follows from Lemma \ref{lemma:rect_const_ball}, where $\xvec, \zvec \in \partial B(\vec{0}, \|\xvec\|)$ are (respectively) the points $\xvec, \yvec$ in that lemma's statement.
    \end{enumerate}
\end{proof}

We have now presented all technical lemmas that will be employed in our proof of Theorem \ref{theorem:consistent_robust_general}. Before moving on to this proof in the next section, we first provide several immediate corollaries of the preceding lemmas characterizing various steps of Algorithm \ref{alg:interp}.

\begin{corollary} \label{corollary:xt_zt}
    In the specification of $\interp$ (Algorithm \ref{alg:interp}), if $x_t$ is determined by Line \ref{algline:notadvice_general}, then $\|\xvec_t - \zvec_t\| \leq \|\svec_t - \svec_{t-1}\|$.
\end{corollary}
\begin{proof}
    This follows immediately from Lemma \ref{lemma:triangle_end_balls} with the lemma's $\avec$, $\bvec$, $\cvec$, and $r$ respectively chosen as $\tilde{\xvec}_t$, $\svec_{t-1}$, $\svec_t$ and $\|\zvec_t - \svec_{t-1}\|$.
\end{proof}

\begin{corollary} \label{corollary:lipschitz_proj_bound}
    In the specification of $\interp$ (Algorithm \ref{alg:interp}),
    $$\|\yvec_t - \xvec_{t-1}\| \leq k(\calX)\|\tilde{\xvec}_t - \tilde{\xvec}_{t-1}\|.$$
\end{corollary}
\begin{proof}
    This follows by definition of the Lipschitz constant $k(\calX)$ of the radial retraction, and the observation that $\yvec_t$ (respectively $\xvec_{t-1}$) is the radial retraction of $\tilde{\xvec}_t$ (respectively $\tilde{\xvec}_{t-1}$) onto the ball $B(\svec_{t-1}, \|\xvec_{t-1} - \svec_{t-1}\|)$.
\end{proof}

\begin{corollary} \label{corollary:gen_proj_bound}
    In the specification of $\interp$ (Algorithm \ref{alg:interp}), 
    $$\|\yvec_t - \xvec_{t-1}\| + \|\tilde{\xvec}_t - \yvec_t\| \leq \mu(\calX)\|\tilde{\xvec}_t - \tilde{\xvec}_{t-1}\| + \|\tilde{\xvec}_{t-1} - \xvec_{t-1}\|.$$
\end{corollary}
\begin{proof}
    This follows immediately from Lemma \ref{lemma:ball_projection_lemma} with $\xvec, \yvec, \wvec,$ and $r$ in the lemma's statement chosen respectively as $\tilde{\xvec}_{t-1}$, $\tilde{\xvec}_t$, $\svec_{t-1}$, and $\|\xvec_{t-1} - \svec_{t-1}\|$, which in turn yields $\hat{\yvec} = \yvec_t$ and $\hat{\xvec} = \xvec_{t-1}$.
\end{proof}

\subsection{Proof of bicompetitive bound}

We prove the bicompetitive bound of Theorem \ref{theorem:consistent_robust_general} in two parts: we will first show the competitive ratio with respect to $\adv$, and will follow with the competitive ratio with respect to $\rob$. Both results proceed via potential function arguments: the first uses the potential function $\|\tilde{\xvec}_t - \xvec_t\|$, and the second uses the potential function $c\|\xvec_t - \svec_t\|$ (with $c$ to be defined later on). The robustness and consistency claim then follows immediately from the bicompetitive bound and the observation in Appendix \ref{appendix:bicompetitive_robust_consistent}.
\\[1em]
\noindent\textbf{Proof of competitiveness with respect to $\adv$.} We define ``phases'' of the algorithm as follows: if $\xvec_t$ is determined by line \ref{algline:advice_general} of the algorithm, then the advice is in the ``$\adv$'' phase. Otherwise, if $\xvec_t$ is determined by line \ref{algline:notadvice_general}, then the advice is in the ``$\rob$'' phase. We refer to the time indices in which the algorithm is in the ``$\rob$'' phase as $R_1, \ldots, R_k \in [T]$ (where $k \leq T$, and $R_1 < \cdots < R_k$ are in increasing order). If the algorithm is never in the ``$\rob$'' phase, then $\xvec_t = \tilde{\xvec}_t$ $\forall t \in [T]$, and thus $\interp$ is 1-competitive with respect to $\adv$. Thus we restrict to the case that there is at least one time index in which the algorithm is in the ``$\rob$'' phase. By design, for each $j \in [k]$, $\cost_{\rob}(1, R_j) \leq \delta\cdot \cost_{\adv}(1, R_j)$.

Now we break into two cases depending on the phase. First, suppose that $\interp$ is in the ``$\adv$'' phase. This means that $\xvec_t = \tilde{\xvec}_t$. Then
\begin{align} 
    f_t(\xvec_t) + \|\xvec_t - \xvec_{t-1}\| + \|\tilde{\xvec}_t - \xvec_t\| &= f_t(\tilde{\xvec}_t) + \|\tilde{\xvec}_t - \xvec_{t-1}\| \nonumber \\
    &\leq f_t(\tilde{\xvec}_t) + \|\tilde{\xvec}_t - \tilde{\xvec}_{t-1}\| + \|\tilde{\xvec}_{t-1} - \xvec_{t-1}\| \label{eq:consistency_alg_advice_case_general}
\end{align}
follows immediately from the triangle inequality.

Second, consider the case that the algorithm is in the ``$\rob$'' phase. This means that $\xvec_t$ is determined by line \ref{algline:notadvice_general} of the algorithm; and there exists some $\lambda \in [0, 1]$ for which ${\xvec_t = \lambda \svec_t + (1-\lambda)\tilde{\xvec}_t}$. In this case, observe
\begin{align}
    &f_t(\xvec_t) + \|\xvec_t - \xvec_{t-1}\| + \|\tilde{\xvec}_t - \xvec_t\| \nonumber \\
    \leq \quad&\lambda f_t(\svec_t) + (1-\lambda)f_t(\tilde{\xvec}_t) + 2\|\xvec_t - \zvec_t\| + 2\|\zvec_t - \yvec_t\| + \|\yvec_t - \xvec_{t-1}\| + \|\tilde{\xvec}_t - \yvec_t\| \label{eq:consistency_triangle_general} \\
    \leq\quad& 2\cdot\cost_{\rob}(t, t) + 2\gamma\cdot\cost_{\adv}(t, t) + f_t(\tilde{\xvec}_t) + \|\yvec_t - \xvec_{t-1}\| + \|\tilde{\xvec}_t - \yvec_t\| \label{eq:consistency_2bds_general}
\end{align}
where (\ref{eq:consistency_triangle_general}) follows from convexity of $f_t$ and the triangle inequality, and (\ref{eq:consistency_2bds_general}) follows from bounding $\|\xvec_t - \zvec_t\|$ via Corollary \ref{corollary:xt_zt} and $\|\zvec_t - \yvec_t\|$ via line (\ref{algline:z_t-y_t_general}) of the algorithm. Invoking Corollary \ref{corollary:gen_proj_bound} gives the result
\begin{align} 
    &f_t(\xvec_t) + \|\xvec_t - \xvec_{t-1}\| + \|\tilde{\xvec}_t - \xvec_t\| \nonumber \\
    \leq \quad&2\cdot\cost_{\rob}(t, t) + 2\gamma\cdot\cost_{\adv}(t, t) + f_t(\tilde{\xvec}_t) + \mu(\calX)\|\tilde{\xvec}_t - \tilde{\xvec}_{t-1}\| + \|\tilde{\xvec}_{t-1} - \xvec_{t-1}\| \nonumber \\
    \leq \quad& 2\cdot\cost_{\rob}(t, t) + (\mu(\calX) + 2\gamma)\cost_{\adv}(t, t) + \|\tilde{\xvec}_{t-1} - \xvec_{t-1}\| \label{eq:robust_phase_bound_general}
\end{align}

Summing (\ref{eq:consistency_alg_advice_case_general}) and (\ref{eq:robust_phase_bound_general}) over time and noting that the left-hand side $\|\tilde{\xvec}_t - \xvec_t\|$ and right-hand side $\|\tilde{\xvec}_{t-1} - \xvec_{t-1}\|$ telescope, we obtain
\begin{align*}
    &\cost_{\interp}(1, T) \\
    \leq\quad &\sum_{t=1}^T f_t(\xvec_t) + \|\xvec_t - \xvec_{t-1}\| + \|\tilde{\xvec}_T - \xvec_T\|\\
    \leq\quad & \sum_{t \in \{R_j\}_{j=1}^k} 2\cdot \cost_{\rob}(t, t) + (\mu(\calX) + 2\gamma)\cost_{\adv}(t, t) + \sum_{t \in [T] \setminus \{R_j\}_{j=1}^k} \cost_{\adv}(t, t)\\
    &\leq 2\cdot \cost_{\rob}(1, R_k) + (\mu(\calX) + 2\gamma) \cost_{\adv}(1, T) \\
    &\leq 2\delta\cdot \cost_{\adv}(1, R_k) + (\mu(\calX) + 2\gamma) \cost_{\adv}(1, T) \\
    &\leq (\mu(\calX) + \epsilon)\cost_{\adv}(1, T)
\end{align*}
where the second to last inequality follows from the assumption that the algorithm is in the ``$\rob$'' phase at time $R_k$, implying $\cost_{\rob}(1, R_k) \leq \delta\cdot \cost_{\adv}(1, R_k)$; and in the last inequality we use the assumption on the parameters that $2\gamma + 2\delta = \epsilon$. This gives the competitive bound with respect to $\adv$. Note that we can repeat the same argument with truncated time horizon to obtain that $\interp$ is $(\mu(\calX) + \epsilon)$-competitive with respect to $\adv$ at every timestep. \jmlrQED

\noindent\textbf{Proof of competitiveness with respect to $\rob$.} Define the potential function $\phi_t = c\|\xvec_t - \svec_t\|$, with $c > 0$ to be determined later. 

Let $t' \in \{0, \ldots, T\}$ be the last time interval in which the algorithm's decision is determined by line \ref{algline:advice_general} of the algorithm, or equivalently, the greatest $t$ such that $\cost_{\rob}(1, t) \geq \delta\cdot \cost_{\adv}(1, t)$. Applying the competitive bound of $\interp$ with respect to $\adv$ to the subhorizon $t = 1, \ldots, t'$, we have $\cost_{\interp}(1, t') \leq (\mu(\calX) + \epsilon)\cost_{\adv}(1, t')$. By the triangle inequality, and since all algorithms begin at the same starting point $\xvec_0$, we have $\|\tilde{\xvec}_{t'} - \svec_{t'}\| \leq \cost_{\adv}(1, t') + \cost_{\rob}(1, t')$. Putting these together, we have
\begin{align*}
    \cost_{\interp}(1, t') + \phi_{t'} &= \cost_{\interp}(1, t') + c\|\tilde{\xvec}_{t'} - \svec_{t'}\| \\
    &\leq (\mu(\calX) + \epsilon)\cost_{\adv}(1, t') + c(\cost_{\adv}(1, t') + \cost_{\rob}(1, t')) \\
    &\leq \left(\frac{\mu(\calX) + \epsilon + c}{\delta} + c\right)\cost_{\rob}(1, t') \tageq\label{eq:advice_phase_robustness_general}
\end{align*}

Now consider arbitrary $t \in \{t'+1, \ldots, T\}$. We distinguish two cases. First, suppose $\xvec_t = \svec_t$. Then
\begin{align*}
    f_t(\xvec_t) + \|\xvec_t - \xvec_{t-1}\| + \phi_t - \phi_{t-1} &= f_t(\svec_t) + \|\svec_t - \xvec_{t-1}\| + c\|\svec_t - \svec_t\| - c\|\svec_{t-1} - \xvec_{t-1}\| \\
    &\leq f_t(\svec_t) + \|\svec_t - \svec_{t-1}\| + \|\svec_{t-1} - \xvec_{t-1}\| - c\|\svec_{t-1} - \xvec_{t-1}\| \\
    &\leq \cost_{\rob}(t, t) \tageq\label{eq:robustness_case_1}
\end{align*}
where the final inequality holds so long as $c \geq 1$.

On the other hand, suppose $\xvec_t \neq \svec_t$. Observe that
\begin{align*}
    \|\xvec_t - \svec_t\| &\leq \|\zvec_t - \svec_{t-1}\| &\text{by line \ref{algline:notadvice_general} of the algorithm} \\
    &= \|\yvec_t - \svec_{t-1}\| - \gamma\cdot\cost_{\adv}(t, t) &\text{by line \ref{algline:z_t-y_t_general} of the algorithm and $\xvec_t \neq \svec_t$} \\
    &\leq \|\xvec_{t-1} - \svec_{t-1}\| - \gamma\cdot\cost_{\adv}(t, t) &\text{by line \ref{algline:firstargmax} of the algorithm} \tageq\label{ineq:potl_diff}
\end{align*}
Then noting that $\xvec_t = \lambda \svec_t + (1-\lambda) \tilde{\xvec}_t$ for some $\lambda \in [0, 1]$, we have
\begin{align*}
    &f_t(\xvec_t) + \|\xvec_t - \xvec_{t-1}\| + \phi_t - \phi_{t-1} \\
    \leq\quad& \lambda f_t(\svec_t) + (1-\lambda)f_t(\tilde{\xvec}_t) + \|\xvec_t - \xvec_{t-1}\| - c\gamma\cdot \cost_{\adv}(t, t) \tageq\label{eq:conv_bnd}\\
    \leq\quad& f_t(\svec_t) + f_t(\tilde{\xvec}_t) + \|\xvec_t - \zvec_t\| + \|\zvec_t - \yvec_t\| + \|\yvec_t - \xvec_{t-1}\| - c\gamma\cdot\cost_{\adv}(t, t) \tageq\label{eq:rob_triangle}\\
    \leq\quad& \cost_{\rob}(t, t) + f_t(\tilde{\xvec}_t) + \gamma\cdot \cost_{\adv}(t, t) + \|\yvec_t - \xvec_{t-1}\| - c\gamma\cdot\cost_{\adv}(t, t) \tageq\label{eq:robustness_case2_triangle_bd} \\
    \leq\quad& \cost_{\rob}(t, t) + (k(\calX) + \gamma - c\gamma)\cost_{\adv}(t, t) \tageq\label{eq:robust_case2_lipschitz} \\
    \leq\quad& \cost_{\rob}(t, t) \tageq\label{eq:robust_case2_final}
\end{align*}
where (\ref{eq:conv_bnd}) follows from convexity and (\ref{ineq:potl_diff}), (\ref{eq:rob_triangle}) follows from the triangle inequality, (\ref{eq:robustness_case2_triangle_bd}) follows from bounding $\|x_t - z_t\|$ via Corollary \ref{corollary:xt_zt} and $\|z_t - y_t\|$ via algorithm line (\ref{algline:z_t-y_t_general}), (\ref{eq:robust_case2_lipschitz}) follows from the observation in Corollary \ref{corollary:lipschitz_proj_bound}, and (\ref{eq:robust_case2_final}) holds so long as $c \geq 1 + \frac{k(\calX)}{\gamma}$.

Thus we set $c = 1 + \frac{k(\calX)}{\gamma}$; summing (\ref{eq:robustness_case_1}) and (\ref{eq:robust_case2_final}) over times $t'+1, \ldots, T$ and adding to (\ref{eq:advice_phase_robustness_general}), we obtain
\begin{align*}
    \cost_{\interp}(1, T) &\leq \left(1 + \frac{k(\calX)}{\gamma} + \frac{\mu(\calX) + \epsilon + 1 + \frac{k(\calX)}{\gamma}}{\delta}\right)\cost_{\rob}(1, t') + \cost_{\rob}(t'+1, T) \\
    &\leq \left(1 + \frac{k(\calX)}{\gamma} + \frac{\mu(\calX) + \epsilon + 1 + \frac{k(\calX)}{\gamma}}{\delta}\right)\cost_{\rob}(1, T).
\end{align*} \jmlrQED

\subsection{Parameter optimization}
We conclude with a brief comment on the optimal selection of parameters $\gamma, \delta$ for $\interp$.
If we minimize the competitive bound of $\interp$ with respect to $\rob$ over parameters $\gamma, \delta > 0$ satisfying $2\gamma + 2\delta = \epsilon$, then we obtain the following $\mathcal{O}(\frac{1}{\epsilon^2})$ bound on the competitive ratio with respect to $\rob$ (with arguments of $\mu(\calX), k(\calX)$ suppressed):
$$3 + \frac{2(\epsilon + k(\calX)(4 + \epsilon) + \epsilon\mu(\calX)) + 4\sqrt{k(\calX)(2+\epsilon)(2k(\calX)+\epsilon(1+\epsilon+\mu(\calX)))}}{\epsilon^2}$$
which is obtained by setting
$$\gamma = \frac{\sqrt{k(\calX)(2+\epsilon)(2k(\calX)+\epsilon(1+\epsilon+\mu(\calX)))} - k(\calX)(2+\epsilon)}{2(1-k(\calX)+\epsilon+\mu(\calX))}$$
and
$$\delta = \frac{\epsilon}{2} - \gamma.$$
With parameters chosen optimally thus, $\interp$ is $(\mu(\calX)+\epsilon, \calO(\epsilon^{-2}))$-bicompetitive. Moreover, even if $\mu(\calX)$ and $k(\calX)$ are not known exactly, simply setting $\gamma = \delta = \frac{\epsilon}{4}$ gives an (up to a constant factor) identical $(\mu(\calX)+\epsilon, \calO(\epsilon^{-2}))$-bicompetitiveness.

\section{Proof of Theorem \ref{theorem:bounded_bicompetitive} \label{appendix:bounded_bicompetitive}}

We prove Theorem \ref{theorem:bounded_bicompetitive} in two parts: we first prove the competitive ratio of $\binterp$ with respect to $\adv$, and then we prove the competitive ratio with respect to $\rob$. The robustness and consistency claim then follows immediately from the bicompetitive bound and the observation in Appendix \ref{appendix:bicompetitive_robust_consistent}.
\\[1em]
\noindent\textbf{Proof of competitiveness with respect to $\adv$.} We define ``phases'' of the algorithm as follows: if $\xvec_t$ is determined by line \ref{bd_algline:advice_general} of the algorithm, then the advice is in the ``$\adv$'' phase. Otherwise, if $\xvec_t$ is determined by line \ref{bd_algline:notadvice_general}, then the advice is in the ``$\rob$'' phase. We refer to the time indices in which the algorithm is in the ``$\rob$'' phase as $R_1, \ldots, R_k \in [T]$ (where $k \leq T$, and $R_1 < \cdots < R_k$ are in increasing order). If the algorithm is never in the ``$\rob$'' phase, then $\xvec_t = \adv_t$ $\forall t \in [T]$, and thus $\binterp$ is 1-competitive with respect to $\adv$. Thus we restrict to the case that there is at least one time index in which the algorithm is in the ``$\rob$'' phase. By design, for each $j \in [k]$, $\cost_{\rob}(1, R_j) \leq \delta\cdot \cost_{\adv}(1, R_j)$.

Now we break into two cases depending on the phase. First, suppose the $\binterp$ is in the ``$\adv$'' phase. This means that $\xvec_t = \tilde{\xvec}_t$. Then
\begin{align} 
    f_t(\xvec_t) + \|\xvec_t - \xvec_{t-1}\| + \|\tilde{\xvec}_t - \xvec_t\| &= f_t(\tilde{\xvec}_t) + \|\tilde{\xvec}_t - \xvec_{t-1}\| \nonumber\\
    &\leq f_t(\tilde{\xvec}_t) + \|\tilde{\xvec}_t - \tilde{\xvec}_{t-1}\| + \|\tilde{\xvec}_{t-1} - \xvec_{t-1}\|\label{bd_eq:consistency_alg_advice_case_general}
\end{align}
follows immediately from the triangle inequality.

Second, consider the case that the algorithm is in the ``$\rob$'' phase. This means that $\xvec_t$ is determined by line \ref{bd_algline:notadvice_general} of the algorithm; and there exists some $\lambda \in [0, 1]$ for which $\xvec_t = \lambda \svec_t + (1-\lambda)\tilde{\xvec}_t$. In this case, observe
\begin{align}
    &f_t(\xvec_t) + \|\xvec_t - \xvec_{t-1}\| + \|\tilde{\xvec}_t - \xvec_t\| \nonumber \\
    \leq \quad&\lambda f_t(\svec_t) + (1-\lambda)f_t(\tilde{\xvec}_t) + 2\|\xvec_t - \yvec_t\| + \|\yvec_t - \xvec_{t-1}\| + \|\tilde{\xvec}_t - \yvec_t\| \label{bd_eq:consistency_triangle_general} \\
    \leq\quad& f_t(\svec_t) + f_t(\tilde{\xvec}_t) + 2\gamma\cdot \cost_{\adv}(t, t) + \|\yvec_t - \xvec_{t-1}\| + \|\tilde{\xvec}_t - \yvec_t\| \label{bd_eq:xt_yt}
\end{align}
where (\ref{bd_eq:consistency_triangle_general}) follows from convexity of $f_t$ and the triangle inequality, (\ref{bd_eq:xt_yt}) follows via algorithm line \ref{bd_algline:notadvice_general}. Observing that $\xvec_{t-1} = \nu\tilde{\xvec}_{t-1} + (1-\nu)\svec_{t-1}$, we can use the triangle inequality to obtain
\begin{equation} \label{ineq:yt_xtmin1_triange}
    \|\yvec_t - \xvec_{t-1}\| \leq \nu\|\tilde{\xvec}_t - \tilde{\xvec}_{t-1}\| + (1-\nu)\|\svec_t - \svec_{t-1}\|.
\end{equation} 
Moreover, observe
\begin{align*}
    \|\tilde{\xvec}_t - \yvec_t\| &= (1-\nu)\|\tilde{\xvec}_t - \svec_t\| \\
    &\leq (1-\nu)(\|\tilde{\xvec}_t - \tilde{\xvec}_{t-1}\| + \|\tilde{\xvec}_{t-1} - \svec_{t-1}\| + \|\svec_t - \svec_{t-1}\|) \\
    &= (1-\nu)(\|\tilde{\xvec}_t - \tilde{\xvec}_{t-1}\| + \|\svec_t - \svec_{t-1}\|) + \|\tilde{\xvec}_{t-1} - \xvec_{t-1}\| \tageq\label{ineq:xtilt_yt_bnd}
\end{align*}
where the final equality follows by definition of $\nu$. Applying (\ref{ineq:yt_xtmin1_triange}) and (\ref{ineq:xtilt_yt_bnd}) to (\ref{bd_eq:xt_yt}), we obtain
\begin{align} 
    &f_t(\xvec_t) + \|\xvec_t - \xvec_{t-1}\| + \|\tilde{\xvec}_t - \xvec_t\| \nonumber \\
    \leq \quad& 2\cdot \cost_{\rob}(t, t) + (1 + 2\gamma)\cost_{\adv}(t, t) + \|\tilde{\xvec}_{t-1} - \xvec_{t-1}\| \label{bd_eq:robust_phase_bound_general}
\end{align}

Summing (\ref{bd_eq:consistency_alg_advice_case_general}) and (\ref{bd_eq:robust_phase_bound_general}) over time and noting that the left-hand side $\|\tilde{\xvec}_t - \xvec_t\|$ and right-hand side $\|\tilde{\xvec}_{t-1} - \xvec_{t-1}\|$ telescope, we obtain
\begin{align*}
    &\cost_{\binterp}(1, T) \\
    \leq\quad &\sum_{t=1}^T f_t(\xvec_t) + \|\xvec_t - \xvec_{t-1}\| + \|\tilde{\xvec}_T - \xvec_T\|\\
    \leq\quad & \sum_{t \in \{R_j\}_{j=1}^k} 2\cdot \cost_{\rob}(t, t) + (1 + 2\gamma)\cost_{\adv}(t, t) + \sum_{t \in [T] \setminus \{R_j\}_{j=1}^k} \cost_{\adv}(t, t)\\
    &\leq 2\cdot \cost_{\rob}(1, R_k) + (1 + 2\gamma) \cost_{\adv}(1, T) \\
    &\leq 2\delta\cdot \cost_{\adv}(1, R_k) + (1 + 2\gamma) \cost_{\adv}(1, T) \\
    &\leq (1 + \epsilon)\cost_{\adv}(1, T)
\end{align*}
where the second to last inequality follows from the assumption that the algorithm is in the ``$\rob$'' phase at time $R_k$, implying $\cost_{\rob}(1, R_k) \leq \delta\cdot \cost_{\adv}(1, R_k)$; and in the last inequality we use the assumption on the parameters that $2\gamma + 2\delta = \epsilon$. This gives the competitive bound with respect to $\adv$. Note that we can repeat the same argument with truncated time horizon to obtain that $\binterp$ is $(1 + \epsilon)$-competitive with respect to $\adv$ at every timestep. \jmlrQED

\noindent\textbf{Proof of competitiveness with respect to $\rob$.} Define the potential function $\phi_t = c\frac{\|\xvec_t - \svec_t\|}{\|\tilde{\xvec}_t - \svec_t\|}$, with $c > 0$ to be determined later (we set $\phi_t \coloneqq 0$ in the case that $\tilde{\xvec}_t = \svec_t$).

Let $t' \in \{0, \ldots, T\}$ be the last time interval in which the algorithm's decision is determined by line \ref{bd_algline:advice_general} of the algorithm, or equivalently, the greatest $t$ such that $\cost_{\rob}(1, t) \geq \delta\cdot\cost_{\adv}(1, t)$. Applying the competitive bound of $\binterp$ with respect to $\adv$ to the subhorizon $t = 1, \ldots, t'$, we have $\cost_{\binterp}(1, t') \leq (1 + \epsilon)\cost_{\adv}(1, t')$. Thereby we obtain
\begin{align*}
    \cost_{\binterp}(1, t') + \phi_{t'} &\leq (1 + \epsilon)\cost_{\adv}(1, t') + c \\
    &\leq \frac{1 + \epsilon}{\delta}\cost_{\rob}(1, t') + c. \tageq\label{bd_eq:advice_phase_robustness_general}
\end{align*}
where in the first inequality we have used the fact that $\|\xvec_t - \svec_t\| \leq \|\tilde{\xvec}_t - \svec_t\|$ for all $t$.

Now consider arbitrary $t \in \{t'+1, \ldots, T\}$. We distinguish two cases. First, suppose $\xvec_t = \svec_t$ and $\tilde{\xvec}_{t-1} \neq \svec_{t-1}$. Then
\begin{align*}
    f_t(\xvec_t) + \|\xvec_t - \xvec_{t-1}\| + \phi_t - \phi_{t-1} &= f_t(\svec_t) + \|\svec_t - \xvec_{t-1}\| - c\frac{\|\xvec_{t-1} - \svec_{t-1}\|}{\|\tilde{\xvec}_{t-1} - \svec_{t-1}\|} \\
    &\leq f_t(\svec_t) + \|\svec_t - \svec_{t-1}\| + \|\xvec_{t-1} - \svec_{t-1}\| - c\frac{\|\xvec_{t-1} - \svec_{t-1}\|}{\|\tilde{\xvec}_{t-1} - \svec_{t-1}\|} \\
    &\leq \cost_{\rob}(t, t) \tageq\label{bd_eq:robustness_case_1}
\end{align*}
where the final inequality holds so long as $c \geq D$, by $D$-boundedness of $\adv$ and $\rob$. Clearly (\ref{bd_eq:robustness_case_1}) will also hold in the case that $\tilde{\xvec}_{t-1} = \svec_{t-1}$, since this will imply $\xvec_{t-1} = \svec_{t-1}$. 

On the other hand, suppose $\xvec_t \neq \svec_t$. Thus we can assume that $\tilde{\xvec}_t \neq \svec_t$ and $\tilde{\xvec}_{t-1} \neq \svec_{t-1}$. First, note that 
\begin{align*}
    \frac{\|\xvec_t - \svec_t\|}{\|\tilde{\xvec}_t - \svec_t\|} &= \frac{\|\yvec_t - \svec_t\| - \gamma\cdot \cost_{\adv}(t, t)}{\|\tilde{\xvec}_t - \svec_t\|} \tageq\label{eq:yt_xt_bd}\\
    &= \nu - \frac{\gamma\cdot \cost_{\adv}(t, t)}{\|\tilde{\xvec}_t - \svec_t\|} \\
    &\leq \frac{\|\xvec_{t-1} - \svec_{t-1}\|}{\|\tilde{\xvec}_{t-1} - \svec_{t-1}\|} - \frac{\gamma\cdot \cost_{\adv}(t, t)}{D} \tageq\label{eq:nu_and_Dbound}
\end{align*}
where (\ref{eq:yt_xt_bd}) follows from line \ref{bd_algline:notadvice_general} of the algorithm and $\xvec_t \neq \svec_t$, and (\ref{eq:nu_and_Dbound}) follows by definition of $\nu$ and the $D$-boundedness of $\adv, \rob$.

Then noting that by convexity, $\xvec_t = \lambda \svec_t + (1-\lambda) \tilde{\xvec}_t$ for some $\lambda \in [0, 1]$, we have
\begin{align*}
    &f_t(\xvec_t) + \|\xvec_t - \xvec_{t-1}\| + \phi_t - \phi_{t-1} \\
    \leq\quad& \lambda f_t(\svec_t) + (1-\lambda)f_t(\tilde{\xvec}_t) + \|\xvec_t - \xvec_{t-1}\| - c\frac{\gamma\cdot \cost_{\adv}(t, t)}{D} \tageq\label{bd_eq:conv_bnd}\\
    \leq\quad& f_t(\svec_t) + f_t(\tilde{\xvec}_t) + \|\xvec_t - \yvec_t\| + \|\yvec_t - \xvec_{t-1}\| - c\frac{\gamma\cdot \cost_{\adv}(t, t)}{D} \tageq\label{bd_eq:rob_triangle}\\
    \leq\quad& f_t(\svec_t) + f_t(\tilde{\xvec}_t) + \gamma\cdot \cost_{\adv}(t, t) + \nu\|\tilde{\xvec}_t - \tilde{\xvec}_{t-1}\| + (1-\nu)\|\svec_t - \svec_{t-1}\| - c\frac{\gamma\cdot \cost_{\adv}(t, t)}{D} \tageq\label{bd_eq:robustness_case2_triangle_bd} \\
    \leq\quad& \cost_{\rob}(t, t) + \left(1 + \gamma - \frac{c\gamma}{D}\right)\cost_{\adv}(t, t) \\
    \leq\quad& \cost_{\rob}(t, t) \tageq\label{bd_eq:robust_case2_final}
\end{align*}
where (\ref{bd_eq:conv_bnd}) follows from convexity and (\ref{eq:nu_and_Dbound}), (\ref{bd_eq:rob_triangle}) follows from the triangle inequality, and (\ref{bd_eq:robustness_case2_triangle_bd}) follows from (\ref{ineq:yt_xtmin1_triange}) and line \ref{bd_algline:notadvice_general} of the algorithm. The final inequality (\ref{bd_eq:robust_case2_final}) holds as long as $c \geq D + \frac{D}{\gamma}$.

Thus we set $c = D + \frac{D}{\gamma}$; summing (\ref{bd_eq:robustness_case_1}) and (\ref{bd_eq:robust_case2_final}) over times $t'+1, \ldots, T$ and adding to (\ref{bd_eq:advice_phase_robustness_general}), we obtain
\begin{align*}
    \cost_{\interp}(1, T) &\leq \frac{1 + \epsilon}{\delta}\cost_{\rob}(1, t') + D + \frac{D}{\gamma} + \cost_{\rob}(t'+1, T) \\
    &\leq \frac{1 + \epsilon}{\delta}\cost_{\rob}(1, T) + D + \frac{D}{\gamma} \\
    &\leq \left(D + \frac{D}{\gamma} + \frac{1+\epsilon}{\delta}\right)\cost_{\rob}(1, T)
\end{align*}
where in the final inequality we have used the assumption that $\cost_{\rob} \geq 1$. \jmlrQED

\subsection{Parameter optimization}

To conclude, we briefly comment on the optimal selection of parameters $\gamma, \delta$ for $\binterp$.
Optimizing the competitive bound of $\binterp$ with respect to $\rob$ over those $\gamma, \delta > 0$ satisfying $2\gamma + 2\delta = \epsilon$, we obtain the following $\calO(\frac{D}{\epsilon})$-competitive bound with respect to $\rob$:
$$2 + D + \frac{2(1+D) + 4\sqrt{D(1+\epsilon)}}{\epsilon}$$
which is obtained by setting
$$\gamma = \frac{D\epsilon}{2(D + \sqrt{D(1+\epsilon)}}$$
and
$$\delta = \frac{\epsilon}{2} - \gamma.$$
With parameters chosen optimally thus, $\binterp$ is $(1+\epsilon, \calO(D\epsilon^{-1}))$-bicompetitive. Moreover, even if $D$ is not known exactly \emph{a priori}, simply setting $\gamma = \delta = \frac{\epsilon}{4}$ gives an (up to a constant factor) identical $(1+\epsilon, \calO(D\epsilon^{-1}))$-bicompetitiveness.

\section{Robustness and consistency corollaries of Theorems \ref{theorem:consistent_robust_general} and \ref{theorem:bounded_bicompetitive} \label{appendix:laundry_list}}

In this section, we detail the upper bounds on robustness and consistency resulting from Theorems \ref{theorem:consistent_robust_general} and \ref{theorem:bounded_bicompetitive} on $\cfc$ and each of its subclasses defined in Appendix \ref{appendix:cfc_subclasses}. Each of these corollaries follows immediately upon instantiating the robust algorithm $\rob$ provided as input to $\interp$ (Algorithm \ref{alg:interp}) or $\binterp$ (Algorithm \ref{alg:binterp}) with a competitive algorithm whose competitive ratio is listed in Table \ref{table:sota_extended}. We begin with the corollaries of Theorem \ref{theorem:consistent_robust_general}.

\begin{corollary}
    \begin{enumerate}[(i)]
    \itemsep0em 
        \item $\interp$ (Algorithm \ref{alg:interp}) with $\rob$ chosen as the functional Steiner point algorithm (\cite{sellke_chasing_2020}) is $(\mu(\R^d, \|\cdot\|)+\epsilon)$-consistent and $\calO(\frac{d}{\epsilon^2})$-robust for $\cfc$ and $\cbc$ on $\R^d$ with any norm.
        
        \item $\interp$ (Algorithm \ref{alg:interp}) with $\rob$ chosen as the low-dimensional chasing algorithm of \cite{argue_dimension-free_2020} is $(\sqrt{2}+\epsilon)$-consistent and $\calO(\frac{k}{\epsilon^2})$-robust for $\kcbc$ on $(\R^d, \|\cdot\|_{\ell^2})$.
        
        \item $\interp$ (Algorithm \ref{alg:interp}) with $\rob$ chosen as the greedy algorithm (\cite{zhang_revisiting_2021}) is $(\mu(\calX)+\epsilon)$-consistent and $\calO(\frac{1}{\alpha\epsilon^2})$-robust for $\acfc$ on any normed vector space $\calX$.
        
        \item $\interp$ (Algorithm \ref{alg:interp}) with $\rob$ chosen as the greedy OBD algorithm (\cite{lin_personal_2022}) is $(\sqrt{2}+\epsilon)$-consistent and $\calO(\frac{1}{\alpha^{1/2}\epsilon^2})$-robust for $\acfc$ on $(\R^d, \|\cdot\|_{\ell^2})$.
        
        \item $\interp$ (Algorithm \ref{alg:interp}) with $\rob$ chosen as the Move towards Minimizer algorithm (\cite{argue_dimension-free_2020}) is $(\sqrt{2}+\epsilon)$-consistent and $\calO(\frac{2^{\gamma/2}\kappa}{\epsilon^{2}})$-robust for $\kgcfc$ on $(\R^d, \|\cdot\|_{\ell^2})$.
    \end{enumerate}
    In particular, each of the consistency bounds is $\sqrt{2}+\epsilon$ in the case that the decision space is Hilbert.
\end{corollary}

We now present the corollaries of Theorem \ref{theorem:bounded_bicompetitive}.

\begin{corollary} \label{cor:bdinterp_ubs}
    In each of the following, suppose that $(\adv, \rob)$ are $D$-bounded and $\cost_\rob \geq 1$.
    \begin{enumerate}[(i)]
    \itemsep0em 
        \item $\binterp$ (Algorithm \ref{alg:binterp}) with $\rob$ chosen as the functional Steiner point algorithm (\cite{sellke_chasing_2020}) is $(1+\epsilon)$-consistent and $\calO(\frac{dD}{\epsilon})$-robust for $\cfc$ and $\cbc$ on $\R^d$ with any norm. 
        
        \item $\binterp$ (Algorithm \ref{alg:binterp}) with $\rob$ chosen as the low-dimensional chasing algorithm of \cite{argue_dimension-free_2020} is $(1+\epsilon)$-consistent and $\calO(\frac{kD}{\epsilon})$-robust for $\kcbc$ on $(\R^d, \|\cdot\|_{\ell^2})$.
        
        \item $\binterp$ (Algorithm \ref{alg:binterp}) with $\rob$ chosen as the greedy algorithm (\cite{zhang_revisiting_2021}) is $(1+\epsilon)$-consistent and $\calO(\frac{D}{\alpha\epsilon})$-robust for $\acfc$ on any normed vector space. 
        
        \item $\binterp$ (Algorithm \ref{alg:binterp}) with $\rob$ chosen as the Greedy OBD algorithm (\cite{lin_personal_2022}) is $(1+\epsilon)$-consistent and $\calO(\frac{D}{\alpha^{1/2}\epsilon})$-robust for $\acfc$ on $(\R^d, \|\cdot\|_{\ell^2})$. \label{cor:yiheng_tight}
        
        \item $\binterp$ (Algorithm \ref{alg:binterp}) with $\rob$ chosen as the Move towards Minimizer algorithm (\cite{argue_dimension-free_2020}) is $(1+\epsilon)$-consistent and $\calO(\frac{2^{\gamma/2}\kappa D}{\epsilon})$-robust for $\kgcfc$ on $(\R^d, \|\cdot\|_{\ell^2})$.
    \end{enumerate}
\end{corollary}

\end{document}